\documentclass{article} 
\usepackage{iclr2027_conference,times}
\usepackage{wrapfig}

\usepackage[T1]{fontenc}
\usepackage{amsmath,amssymb,mathtools,amsthm}
\usepackage{bm} 
\usepackage{graphicx}
\usepackage{microtype}
\usepackage{hyperref}
\usepackage{url}
\usepackage[font=small,skip=4pt]{caption}
\usepackage{placeins} 

\newcommand{\E}{\mathbb E}
\newcommand{\R}{\mathbb R}
\newcommand{\norm}[1]{\left\lVert #1\right\rVert}

\newcommand{\abs}[1]{\left\lvert #1\right\rvert}

\newcommand{\conj}[1]{\overline{#1}}

\newcommand{\nfit}{512}            
\newcommand{\nperdataset}{16}      
\newcommand{\nheldout}{80}         
\newcommand{\ndirs}{50}            
\newcommand{\horizon}{128}         
\newcommand{\nreduce}{200}         
\newcommand{\dcawindow}{16}        
\newcommand{\nrunssmall}{5}        
\newcommand{\nrunslarge}{3}        

\newtheorem{theorem}{Theorem}[section] 
\newtheorem{lemma}[theorem]{Lemma}      
\newtheorem{corollary}[theorem]{Corollary}
\title{A Mathematical Theory of the Jacobian Lens}

\author{
  Mitchell Ostrow \\
  Department of Brain and Cognitive Sciences \\
  Massachusetts Institute of Technology \\
  \texttt{ostrow@mit.edu} \\
  \And
  Daoyuan Qian \\
  Department of Brain and Cognitive Sciences \\
  Massachusetts Institute of Technology \\
  \AND
  Toni J.B. Liu \\
  Department of Physics \\
  Cornell University
  \And
  Ila Fiete \\
Department of Brain and Cognitive Sciences \\
  Massachusetts Institute of Technology \\
}

\iclrfinalcopy 
\begin{document}

\maketitle
\lhead{Preprint} 

\begin{abstract}
Recent work from Anthropic introduced a new interpretability tool, called the Jacobian lens (J-lens, \citealp{gurnee2026verbalizable}). This method identified linear directions in the hidden states of frontier language models that had properties reminiscent of working-memory-like representations. 
However, it is unclear what mechanistic properties of such features make them detectable by the J-lens. 
To this end, we introduce a mathematical theory of the J-lens and show that it preferentially identifies \emph{coherent} features, directions that act like leading, spatially coherent finite-time Lyapunov modes. Small perturbations along these directions amplify across layers, remain roughly constant across token positions, and are consistent across different contexts. 
While we are agnostic to whether these properties are necessarily associated with a ``Global Workspace'', the analytically derived preferences of the J-lens nevertheless explain prior empirical results. 
We show that the J-lens identifies coherent features in open-source LLMs more effectively than standard baselines. We subsequently validate our theory more directly on
a Bayes filter and a toy transformer, each optimally predicting next tokens emitted from a Hidden Markov Model (HMM). 
We show how modulating task parameters that change the coherence of features in the Bayes filter correspondingly and predictively modulates the response of the J-lens applied to both model classes. We also introduce a family of J-lenses built to emphasize different temporal scales of information processing in LLMs, which we call $\omega$-J-lenses.
\end{abstract}

\vspace{-4mm}
\section{Introduction}
\vspace{-3mm}

In a recent paper, \citet{gurnee2026verbalizable} introduced a novel interpretability technique that they claimed could identify an internal `workspace' in large language models (LLMs), which had properties reminiscent of the Global Workspace \citep{baars1988cognitive,dehaene1998neuronal}. The object they primarily study is the `Jacobian lens', or `J-lens'. This interpretability tool is designed in the spirit of other `lenses' \citep{nostalgebraist2020logitlens,belrose2023tuned}, which seek to decode transformer hidden state activations into language by applying the unembedding matrix of the model to these activations (potentially after a learned affine map, as in the tuned lens). These lenses benefit interpretability by mapping hidden states to tokens, which are in the space of language and hopefully more interpretable. Decoding intermediate layers this way has shown that predictions are refined gradually across depth \citep{nostalgebraist2020logitlens,belrose2023tuned}, that multilingual models pass through an English-like latent space before producing the target language \citep{wendler2024llamas}, that false few-shot demonstrations derail predictions only after a critical layer \citep{halawi2024overthinking}, and that a single hidden state already carries information about tokens several positions ahead \citep{pal2023future}. The trajectory of tuned-lens predictions across layers can also detect malicious inputs \citep{belrose2023tuned}.

The Jacobian lens (J-lens, \citealp{gurnee2026verbalizable}) continues this line of work by introducing a new transfer operator onto the residual stream vector before the application of the unembedding matrix and mapping to tokens. This operator is the Jacobian of the transformer, averaged over tokens. In their paper, the motivation for the usage of the Jacobian $\bm K$ is based on standard results in calculus: A Jacobian describes how a small perturbation in the input of a function results in a change in the output. When the Jacobian is studied in an LLM, it provides directions in the residual stream that are ``verbalizable'' \citep{gurnee2026verbalizable}, which is to say that it identifies how perturbations change the linguistic output of the model. However, this property holds only in the local neighborhood of the point on which the Jacobian is defined, and the J-lens is defined by averaging over prompts, tokens, and future positions in the output layer. An average Jacobian has no reason to behave like any of its elements. Indeed, if the Jacobians of a system form a curved manifold, the average Jacobian can easily lie off-manifold. More generally, well-known features that describe the behavior of linear operators (e.g., singular vectors and values, which are nonlinear functions of a matrix) can deviate grossly from those of individual Jacobians under averaging (Appendix \ref{app: avg_jac}). In Figs. \ref{fig: real_jac} and \ref{appfig: real_jac_svecs} we visually demonstrate this phenomenon for four open-source LLMs. What, then, does the J-lens capture? 
\begin{wrapfigure}[24]{r}{0.5\textwidth}
    \centering
    \includegraphics[width=\linewidth]{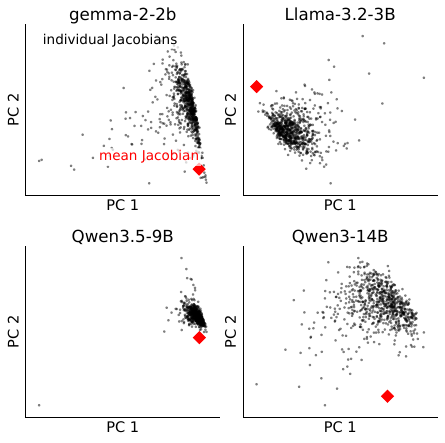}
    \caption{Low-dimensional projection (principal component analysis, PCA) of the top 100 singular values for individual (black) and mean (red) Jacobians of four different LLMs (Appendix~\ref{app: real_jacobians}). Each Jacobian is evaluated with input at an intermediate layer chosen to be halfway through the model and output at the final layer. Prompts are sampled from a range of datasets with different topics (Wikipedia, Python code, math, news, etc.). Jacobian spectra are plotted in Fig. \ref{fig: jac_spectra}.}
    \label{fig: real_jac}
\end{wrapfigure}

Recent work \citep{yan2026short} has attempted to analyze the J-lens as an approximate least squares map from the hidden state to the output using Stein's lemma \citep{stein1981estimation}. However, this argument uses the Gaussian form of Stein's lemma, and transformer hidden states are notoriously non-Gaussian, with heavy-tailed outlier features and a few activations that are orders of magnitude larger than the rest \citep{dettmers2022llmint8,sun2024massive}. Rather, the J-lens is most directly related to a branch of dynamical systems theory called \textit{Lyapunov theory}, which describes how perturbations grow or shrink in spatiotemporal dynamical systems.

\paragraph{The Jacobian lens.}
Let $\bm h_{\ell,t}(x)\in\R^d$ be the activation at layer $\ell$ and token
position $t$ in a fixed sequence $x$, and let $L$ be the final layer. Write $\tau$ to represent a lag in token positions, $H$ for the number of lags averaged over, and $c=(x,t)$ for a context. Define
\begin{equation}
 \begin{aligned}
 \bm K(c,\ell,\tau)&:=\frac{\partial \bm h_{L,t+\tau}(x)}
 {\partial \bm h_{\ell,t}(x)},\\
 \bm G_c(\ell,H)&:=\frac1H\sum_{\tau=1}^{H}\bm K(c,\ell,\tau) \\ \bm J_H(\ell)&:=\E_c\bm G_c(\ell,H).
 \end{aligned}
 \label{eq:jlens}
\end{equation}

The lens is then applied to a state, along with the unembedding matrix:

\begin{equation}
\mathrm{lens}(\bm h_{\ell,t}) = \mathrm{softmax}\big(\bm W_U\,\mathrm{norm}(\bm J_{H}(\ell)\,\bm h_{\ell,t})\big)
\end{equation}
where $\bm W_U$ is the model's unembedding matrix.

\paragraph{Finite-time Lyapunov exponents.}
Let $\delta\bm h(T)$ be a perturbation of a state evolved for $T$ steps under the dynamics. Its Lyapunov exponent \citep{oseledets1968multiplicative,eckmann1985ergodic} measures asymptotic growth of the perturbation:
\begin{equation}
 \chi=\lim_{T\to\infty}\frac1T
 \log\frac{\norm{\delta\bm h(T)}}{\norm{\delta\bm h(0)}},
 \qquad \delta\bm h(0)\ne\bm 0,
 \label{eq:lyapunov}
\end{equation}
when the limit exists. Transformers have finitely many layers, so this limit does not exist and we use only finite-time quantities. At finite $T$, the finite-time Lyapunov exponents (FTLEs) are
\begin{equation}
 \chi_i(T)=\frac1T\log s_i(T),
 \label{eq:finite-lyapunov}
\end{equation}
where $s_i(T)$ are singular values of the full $T$-step Jacobian. For dynamical systems, this is typically computed by chaining with multiplication each one-step Jacobian with iterative renormalization \citep{benettin_lyapunov_1980,engelken2024}.
Spatiotemporal versions also resolve how a perturbation travels across space:
the convective Lyapunov exponent
\citep{deissler1987velocity,kenfack2013convective} measures the growth of a perturbation in a reference frame that is moving in space. Hydrodynamic Lyapunov modes
\citep{eckmann2000hydrodynamic} are long-wavelength Lyapunov vectors whose exponents approach zero as
their wavenumber $k$ approaches zero. 

For transformers, depth is the evolution coordinate ($T$) and token position
the spatial coordinate. Each $\bm K(c,\ell,\tau)$ in Eq.~\eqref{eq:jlens} is then a finite-time Jacobian over $T=L-\ell$ layers, measured in a reference frame shifted by $\tau$ tokens, which is the setting of the convective exponent.
We show here that the J-lens is tightly linked to finite-time Lyapunov theory. Under the assumptions below, the J-lens favors spatially smooth, slowly varying responses, because it by construction integrates out the fast dynamics, implicitly behaving like a low-pass filter in space. With this connection in hand, we show that many of the results in the original paper follow naturally. Note that this is not an argument that it does not identify a Global Workspace per se, but rather to shed light on what this method is designed to capture and why that may or may not align with the philosophical conception of the Global Workspace.

\section{Theory}\label{sec: theory}

We begin by constructing a model that establishes this connection, and subsequently relax assumptions of the model. 

\paragraph{Shared and context-specific modes.}
Suppose our Jacobians can be decomposed as follows:
\begin{equation}
 \begin{aligned}
 \bm K(c,\ell,\tau)
 &=
 \underbrace{\sum_{i} e^{T\chi_i+\tau\gamma_i}\,
 \bm u_i\bm v_i^*}_{\bm S(\ell,\tau):\,\text{shared modes}}
 +
 \underbrace{\sum_{j} e^{T\chi_{cj}+\tau\gamma_{cj}}\,
 \bm u_{cj}\bm v_{cj}^*}
 _{\bm R(c,\ell,\tau):\,\text{context-specific modes}},
 \end{aligned}
 \label{eq:model}
\end{equation}
with $T=L-\ell$ the number of layers between the read-in layer and the output, and every mode quantity depends on $\ell$, which we suppress. $\gamma_i=\mu_i+\mathrm{i}\theta_i$ describes the cross-token evolution dynamics.
The shared input directions $\bm v_i$ and output directions $\bm u_i$ are each orthonormal. Each mode has two finite-time Lyapunov exponents: $\chi_i\in\R$ in depth and $\mu_i\in\R$ in token position, while $\theta_i$ is the mode's rotation per token. Modes with $\theta_i\neq0$ come in complex-conjugate pairs, so that $\bm K$ is real. At a fixed $(\ell,\tau)$, any matrix admits a singular value decomposition (SVD).
Sharing directions across contexts, factoring depth (Fig. \ref{fig: schematic}a) and lag (Fig. \ref{fig: schematic}b) in this
way, and treating the two exponents as independent are additional assumptions of this illustrative model. A Lyapunov mode is a mode that evolves according to a Lyapunov exponent, and hence is defined by its input direction $\bm v_i$.
Averaging over lags and contexts splits the lens into a shared part and an averaged remainder, $\bm J_H=\bm S_H+\E_c\bm R_c$, in which shared mode $i$ is weighted by $e^{T\chi_i}f_H(\gamma_i)$, with
\[
 f_H(\gamma)=\frac1H\sum_{\tau=1}^{H}e^{\tau\gamma}
\]
where $H$ denotes the maximum lag between input and output token positions (Appendix~\ref{app:mode-model}). We also assume the context-specific response is no larger than the shared one: $\E_c\norm{\bm R_c}_2^2\leq M^2$, where $M=\norm{\bm S_H}_2$.

\paragraph{Isolation of slow modes.}
The future averaging $f_H(\gamma)$ suppresses modes that decay rapidly or oscillate quickly. For $\mu_i\leq0$, $|f_H(\gamma_i)|\leq1$. Slow modes with
$H|\mu_i|,H|\theta_i|\ll1$ retain $f_H(\gamma_i)\approx1$, while rapid decay
or oscillation is attenuated. Another way to see this is to consider the Fourier decomposition of a response $y(\tau)$ over the window:
\begin{align}
 \widetilde y_H(\omega)&=\sum_{\tau=1}^{H}
 e^{-\mathrm{i}\omega\tau}y(\tau),\\
 \frac1H\widetilde y_H(0)&=\frac1H\sum_{\tau=1}^{H}y(\tau).
 \label{eq: fourier}
\end{align}
Thus averaging over lags extracts the zero-frequency coefficient of the finite
window, thereby identifying slowly varying responses in space (token position). Similarly, depth weights modes by $e^{T\chi_i}$, thereby favoring growing
or slowly decaying modes. This is not necessarily a strong weighting, especially closer to the output when the Jacobian becomes increasingly aligned with the identity map. 

\paragraph{Incoherent directions are suppressed.}
For a unit input direction $\bm v$ and independent contexts $c,c'$, the averaged remainder is $\norm{(\E_c\bm R_c)\bm v}_2^2=\E_{c,c'}\langle \bm R_c\bm v,\bm R_{c'}\bm v\rangle$, which quantifies the expected overlap between two contexts' responses to the same input (Appendix~\ref{app:mode-model}). Suppose this overlap is small,
$\sup_{\norm{\bm v}=1}\E_{c,c'}|\langle \bm R_c\bm v,\bm R_{c'}\bm v\rangle|
\leq\eta\,\E_c\norm{\bm R_c}_2^2$ with $0\leq\eta < 1$.
Taking the largest growth over unit inputs and assuming bounded scales of the modes,
\begin{equation}
 \boxed{\begin{aligned}
 \norm{\E_c\bm R_c}_2^2
 &=\sup_{\norm{\bm v}=1}\E_{c,c'}\langle \bm R_c\bm v,\bm R_{c'}\bm v\rangle\\
 &\leq\sup_{\norm{\bm v}=1}\E_{c,c'}|\langle \bm R_c\bm v,\bm R_{c'}\bm v\rangle|\\
 &\leq\eta\,\E_c\norm{\bm R_c}_2^2
 \leq\eta M^2=\eta\norm{\bm S_H}_2^2.
 \end{aligned}}
 \label{eq:suppression}
\end{equation}
Thus for small overlap $\eta$, averaging suppresses the incoherent response relative to the shared
response. Since $\bm J_H=\bm S_H+\E_c\bm R_c$, standard perturbation bounds keep the leading singular values and directions of $\bm J_H$ close to those of $\bm S_H$ whenever $\sqrt\eta\,M$ is small compared with the gaps between the shared singular values (Appendix~\ref{app:mode-model}).

\paragraph{Relaxing exact shared modes.}
The above model shows that if spatially (in a transformer, across tokens) and contextually coherent Lyapunov modes exist, these modes dominate in the Jacobian lens. However, exact shared directions need not exist. We show next under a more general construction that the J-lens directions preferentially select Lyapunov modes that (a) maximally grow over depth $T$ or persist across lags (magnitude $a$), (b) are maximally coherent across token lags $\tau$ (lag coherence $\rho_\tau$), and (c) are maximally coherent across different contexts $c$ (context coherence $\rho_c$). 
Using the same SVD notation, write
\begin{equation}
 \bm G_c=\E_\tau\bm K(c,\ell,\tau)  = \sum_j s_{cj}\bm u_{cj}\bm v_{cj}^*,
 \qquad s_{c1}\geq s_{c2}\geq\cdots\geq0.
 \label{eq:full-svd}
\end{equation}
For a unit input direction $\bm v$, define
\begin{equation}
 \begin{aligned}
 a(\bm v)&:=\E_{c,\tau}\norm{\bm K(c,\ell,\tau)\,\bm v},\\[2pt]
 \rho_\tau(\bm v)&:=\frac{\E_c\norm{\bm G_c\bm v}}
                    {\E_{c,\tau}\norm{\bm K(c,\ell,\tau)\,\bm v}}\in[0,1],\\[2pt]
 \rho_c(\bm v)&:=\frac{\norm{\E_{c,\tau}\bm K(c,\ell,\tau)\,\bm v}}
               {\E_c\norm{\bm G_c\bm v}}\in[0,1].
 \end{aligned}
 \label{eq:coherence}
\end{equation}
We call these the \textit{coherence parameters} for directions in the Jacobian lens. We use the 2-norm for each parameter here, but other norms could reasonably apply as well. $a(\bm v)$ is the average response magnitude, $\rho_\tau(\bm v)$ is the lag coherence of a response, and $\rho_c(\bm v)$ is the context coherence. These coherence parameters can be understood intuitively through the lens of spatiotemporal Lyapunov modes: $a$ is the average exponentiated finite-time growth of a direction, $a(\bm v)=\E_{c,\tau}\,e^{T\chi_{c,\tau}(\bm v)}$ with $\chi_{c,\tau}(\bm v)=\frac1T\log\norm{\bm K(c,\ell,\tau)\bm v}$, so it combines growth over layers with persistence across lags. Thus $\frac1T\log a(\bm v)$ is an effective FTLE of $\bm v$. $\rho_\tau$ describes how activity along a direction oscillates across spatial (token) positions, with $\rho_\tau = 1$ indicating no oscillation. Lastly, $\rho_c$ describes how consistently a direction is read in and written out across different contexts (instantiations of the Jacobian), given the same vector. 

We illustrate what each coherence parameter measures in Fig. \ref{fig: schematic}. Intuitively, magnitude measures the total effect of a perturbation along a feature direction across both layers and time, thereby capturing how persistent or relevant it is. Lag coherence measures how the response of a perturbation changes across token positions, for which large values are approximately static, such as the overall topic of a document, whereas incoherent features change rapidly, such as word type. Context coherence measures whether a perturbation along a direction responds in the same way in different contexts, and how prior context modulates its effect. For example, adding a new rule such as ``today is opposite day'' to the context requires reprocessing each feature with this in mind. Context-coherent directions are processed consistently across contexts, while incoherent directions change their behavior. 

\begin{figure}
    \centering
    \includegraphics[width=\linewidth]{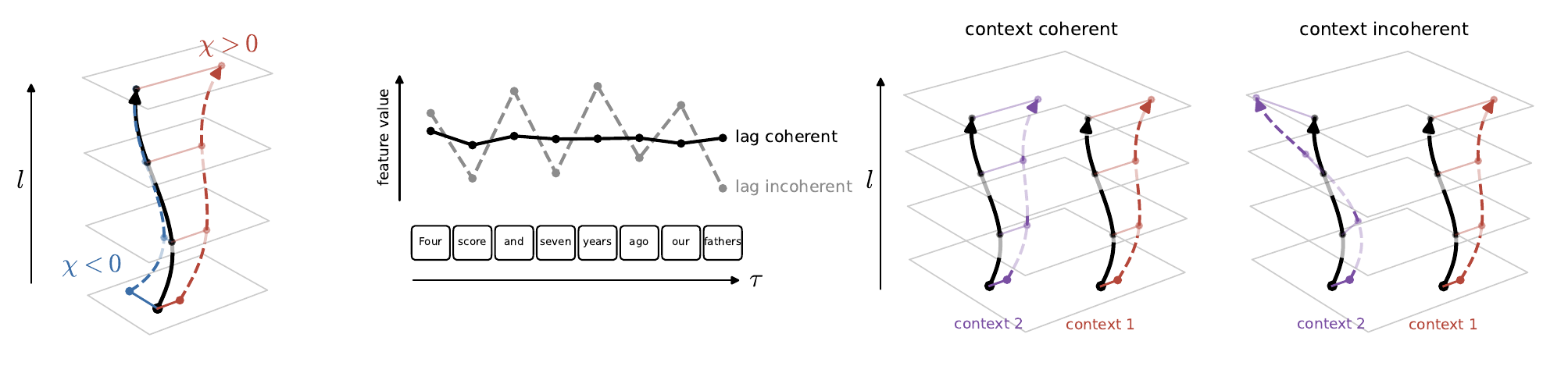}
    \caption{Schematic of the three coherence parameters. 
    \textbf{(left)} Schematic of large $\chi >0$ versus small $\chi <0$ magnitude directions. A perturbation grows across layers in the red direction, but shrinks in the blue direction. 
    \textbf{(middle)} Schematic of lag-coherent (black) versus lag-incoherent (gray) feature directions across token positions. 
    \textbf{(right)} Context-coherent directions evolve similarly across layers in contexts, while context-incoherent directions evolve differently.}
    \label{fig: schematic}
\end{figure}

Define $\rho_c = 0$, $\rho_\tau = 0$ when their denominators are zero. 
The three factors telescope via multiplication to yield $\norm{\bm J_H\bm v}$. The leading singular value of the lens is therefore
\begin{equation}
 \boxed{\sigma_1(\bm J_H)=\max_{\norm{\bm v}=1}\rho_c(\bm v)\,\rho_\tau(\bm v)\,a(\bm v).}
 \label{eq:ranking}
\end{equation}
Subsequent directions maximize the same quantity subject to orthogonality
to prior input directions $\bm v_1,\dots,\bm v_{i-1}$. This shows that the J-lens preferentially upweights directions with the largest combinations of magnitudes $a(\bm v)$ and coherences $\rho_\tau(\bm v)\,\rho_c(\bm v)$. The split depends on the order of averaging. We average over lags first, as in the definition of $\bm G_c$ for a single context. In the notation of Eq.~\eqref{eq:model}, the J-lens selects the modes with the largest growth across layers $\chi_i$ (magnitude), spatial changes that do not oscillate much ($\theta_i \approx 0$), and $\bm u_i$, $\bm v_i$ that change minimally across contexts (context coherence). \citet{gurnee2026verbalizable} found that the J-space summarized only a small portion of the variance. This observation is consistent with the fact that variance explained is unrelated to each of the factors in Eq.~\eqref{eq:ranking}, which we formalize now:

\begin{corollary}\label{cor:variance}
 A direction selected by Jacobian averaging may explain arbitrarily little or arbitrarily much activation variance. Consequently, the singular values of \(\bm J_H\) do not determine the variance explained by its selected directions.
\end{corollary}
The proof is a two-dimensional counterexample and is given in Appendix~\ref{app:corollary-proof}.

\subsection{Mapping coherent modes onto tokens} 
The analysis above characterizes the singular modes selected by Jacobian averaging, whereas the J-space used in the original paper is constructed from token-indexed directions $\bm j_q$ and is therefore not simply the leading singular subspace of $\bm J_H$. These descriptions can nevertheless be connected exactly. Let $\bm J_H=\bm U\bm\Sigma\bm V^*$, and let $\bm w_q$ denote the unembedding vector for token $q$. The corresponding J-lens direction in the input residual stream is
\begin{equation}
\bm j_q:=\bm J_H^*\bm w_q
\label{eq:jlens-token}
=\bm V\bm\Sigma\bm U^*\bm w_q
=\sum_i \sigma_i\langle \bm u_i,\bm w_q\rangle \bm v_i.
\end{equation}
Thus every token-indexed J-lens direction is a mixture of the right singular modes $\bm v_i$: the singular value $\sigma_i$ weights the strength of the corresponding coherent response channel, while $\langle \bm u_i,\bm w_q\rangle$ measures how strongly that channel's downstream response aligns with the token's unembedding direction. This connection matters because the J-space is operationally defined by approximating an activation as a sparse non-negative combination of the directions $\bm j_q$ \citep{gurnee2026verbalizable}. The sparse decomposition therefore selects a small number of semantically labeled mixtures of the coherent modes that survive averaging. In this sense, the singular structure of $\bm J_H$ supplies the underlying dynamical geometry, while the unembedding and sparse non-negative decomposition translate that geometry into the token-level ``concepts'' reported by the J-lens.

\subsection{Families of Jacobian Lenses} \label{subsec: omega lens}

Our analysis suggests that the Jacobian lens can be modified to extract different types of information. Here, we focus on the modulation of lag coherence selectivity. Eq.~\eqref{eq: fourier} highlights that averaging over time points is equivalent to a Fourier projection onto the zero-frequency mode. Modulating the $\omega$ parameter in this projection (or even generalizing to band-pass filtering) induces a family of lenses we call $\omega$-J-lenses, $\bm J_H^{\omega}:=\E_c\frac1H\sum_{\tau=1}^{H}e^{-\mathrm{i}\omega\tau}\bm K(c,\ell,\tau)$, where $\omega \in [-\pi,\pi]$ given the Nyquist limit for sampling frequency $1$ / token. We demonstrate their selectivity in our toy task in Fig.~\ref{fig: omega-lenses}. 

\section{Empirical demonstration on open-source LLMs}

We first show that the J-lens directions preferentially optimize growth across layers, lag coherence, and context coherence in real LLMs. 
For each LLM in Fig.~\ref{fig: real_jac}, we compute J-lens, baseline and random directions, and measure the coherence parameters $a$, $\rho_\tau$ and $\rho_c$ of each direction on held-out text by passing it through the Jacobian (Appendix~\ref{app:jlens_llm}). 
As baselines, we use PCA, slow feature analysis (SFA; \citealp{sfapaper}), and dynamical components analysis (DCA; \citealp{clarkdca}). We present results in Fig.~\ref{fig: jlens_llm}.

The J-lens directions showed a clear monotonic decay in each coherence parameter for each LLM, while other methods were far less ordered. J-lens directions were consistently more coherent than random directions across all models. PCA directions had a surprisingly large magnitude relative to random directions, even matching the J-lens on Qwen3-14B after the fifth direction (Fig.~\ref{fig: coherence_single_llms}). Interestingly, the larger Qwen3-14B model had a much steeper decay across all coherence parameters. Fewer large-magnitude directions suggests that the larger model is more stable in the forward pass, and fewer lag- and context-coherent directions suggest that the larger models have learned to localize the coherent direction a lower-rank subspace. While we do not analyze a functional relationship between these parameters across scales here, these results are intriguing and worth further study. Overall, these results confirm the theory on open-source LLMs and demonstrate that the J-lens is more effective at identifying slow, output-coupled features than other linear methods. 

\begin{figure}
    \centering
    \includegraphics[width=\linewidth]{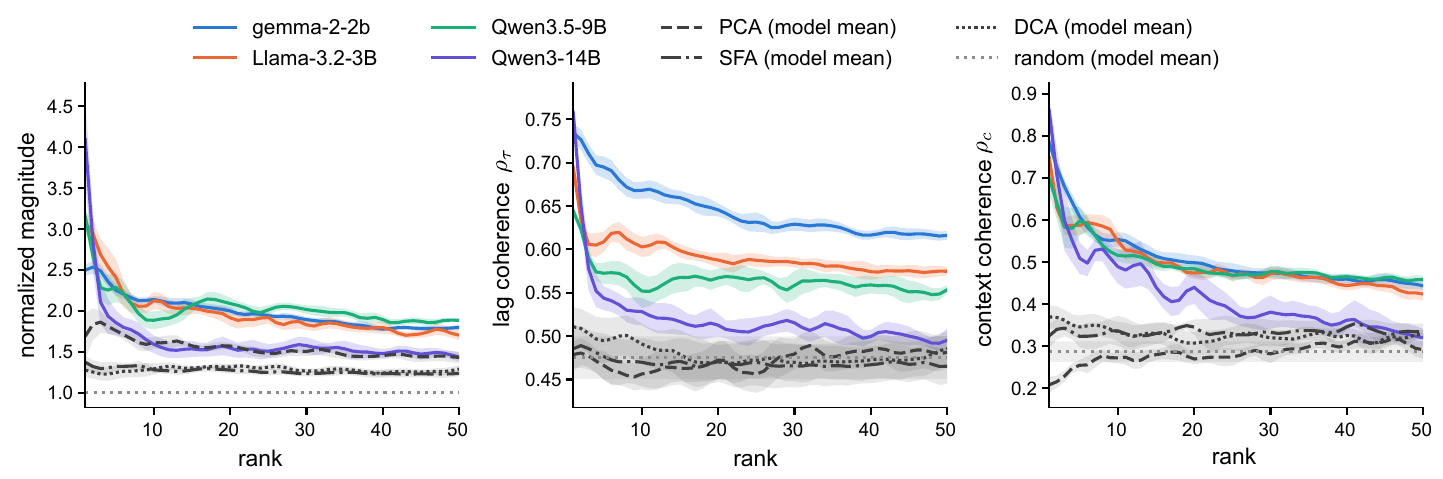}
    \caption{J-lens directions preferentially optimize all coherence factors. 
    \textbf{(left)} Magnitude $a(\bm v)$ of rank-ordered directions, normalized by the mean magnitude of random directions. Colored lines indicate J-lens directions for various models, while black lines indicate baseline directions averaged across models for clarity. Fig.~\ref{fig: coherence_single_llms} shows the baselines for each model separately. 
    \textbf{(middle)} Lag coherence for directions of each method and model, rank-ordered. 
    \textbf{(right)} Context coherence of the directions of each method. 
    Shading is $\pm 1$ standard error. Curves are smoothed for display with a Gaussian kernel of width $\sigma = 1$ rank.
    }
    \label{fig: jlens_llm}
\end{figure}

\section{Toy Model}

We designed a simple toy task to study the J-lens for models performing next-token prediction, and demonstrate its behavior on optimal Bayes filters as well as trained transformers. Our task, inspired by recent work studying how transformers learn to track hidden states \citep{shai2024transformers}, is a hidden Markov model (HMM) with several independent latent states $X^{(i)}$, each arranged on a ring of $n_i$ discrete states: $X^{(i)} \in \{0,1,\dots,n_i -1\}$. We design the dynamics to be the simplest model that can selectively modulate each of the coherence factors via the task parameters. The dynamics of each latent state are (suppressing the latent index for clarity):
\begin{equation}
    X_{t+1} = \begin{cases}
    X_{t} + v + D_t  \mod n\quad \text{w.p. } \lambda \\
    U_t \quad \text{w.p. } 1-\lambda
    \end{cases}
\end{equation}
Here $\lambda$ is called the state persistence parameter, $D_t \in \{-1,0,1\}$ with probabilities $q, 1-2q, q$ and $U_t \sim \mathrm{Uniform}\{0,\dots,n-1\}$. We will see how changing $\lambda$ modulates magnitude $a$ by determining for how many positions a token is useful to estimate the state of the latents. 
The integer $v$ is the velocity of the state on the ring, $\theta = 2\pi v / n$ is its angular velocity, and $D_t$ is a diffusion term. Lag coherence measures whether a state drives all future states in the same way, and $v$ modulates this by increasing misalignment across positions.
With a fully observed state at all timepoints, however, the model is not forced to integrate over time (filtering), which would make the Jacobian $0$ across all timepoints. To force integration, we reveal the states with probability $\varepsilon$, via another random variable $B_t\in\{0,1\}$. Each latent state is combined into a token by concatenating state index into an $n$-ary variable and converting that to a token index $Y_t$, making a total of $1+\prod_i n_i$ tokens, where the additional token is a null token (states unobserved). In the following, we use three latents with $n_i \in \{2,8\}$ for simplicity. We will show how this model setup allows us to selectively modulate the magnitude $a$ (through $\lambda$) and lag coherence $\rho_\tau$ (through $\theta$) of the Jacobian lens vectors. Note that this toy model is easily solved by a transformer, and indeed even the embeddings learn to represent the latent states. The main result of our analysis is that we can \textit{a priori predict the order of the modes and their approximate coherence parameters in the Jacobian lens of the transformer.} 

\subsection{Introducing contextual dependency} \label{sec: keystate}
So far, all tokens across all random sequences for a given model are integrated the same way. This means that the \textit{context coherence} $\rho_c$ of the Jacobian lens vectors cannot be modulated. To introduce this, we add a new variable called a \textit{key} state $W_{t} \in \{-1,+1\}$ where the observed state $Y_t= X_t + \frac{n}{2} \mathbf{1}(W_{t-1} = -1)$, and $P(W_t = 1) = (1 + \psi)/2$, with key mean $\psi\in[0,1]$. This state is always fully observed, forcing the models to look back one timestep to identify how to infer the state of the observed token at the present timestep. This directly modulates context coherence because it forces the read-in and read-out directions of a feature to change dependent on the context.
Each keyed state has its own independent key, making $(1+2\prod_i n_i)\,2^{|\mathcal K|}$ tokens, where $\mathcal K$ is the set of latents with keys.

\section{The Bayes Filter and its Jacobian Lens}

Prior work \citep{shai2024transformers} has found that transformers trained on HMMs learned to represent optimal Bayesian beliefs of the hidden states in their residual stream. We follow this line of work by deriving the Bayes filter for the above HMM and its Jacobian lens, and then showing that the Jacobian lens captures belief states \textit{ordered by their coherence parameters} in the transformer as well. 
Because each state exists on a ring, we account for rotational symmetry by studying the Fourier modes $\phi_j(x) = e^{2\pi \mathrm{i} j x /n}$. The Bayes filter optimally integrates the given evidence into a \textit{belief state}, which is defined as the expected value of a state given the history:
\begin{align}
m^{(j)}_t&:=\E[\phi_j(X_{t+1})\mid Y_1,\dots,Y_t] \\ m^{(j)}_0 &:=0 .
\end{align}

The belief state for a given Fourier mode of a given state is updated by unkeyed evidence as
\begin{equation}
    m^{(j)}_{t} = \alpha_j \big[(1-B_t)m^{(j)}_{t-1} + B_t \phi_j(X_t)\big]
\end{equation}
where $\alpha_j = \lambda c_j e^{\mathrm{i}\theta_j}$, with $\theta_j = j\theta$ the rotation of mode $j$ per step and $c_j = 1 - 2q + 2q\cos(2\pi j/n)$ the diffusion factor (Appendix~\ref{app: bayes}). This update rule is locally linear, either updating with the expected evolution operator $\alpha_j$ if no evidence is observed, and replacing the state with the observation otherwise. The graphical model of the task and the Bayes filter are schematized in Fig. \ref{fig:task schematic}.

\begin{figure}[t]
    \centering
    \includegraphics[width=\linewidth]{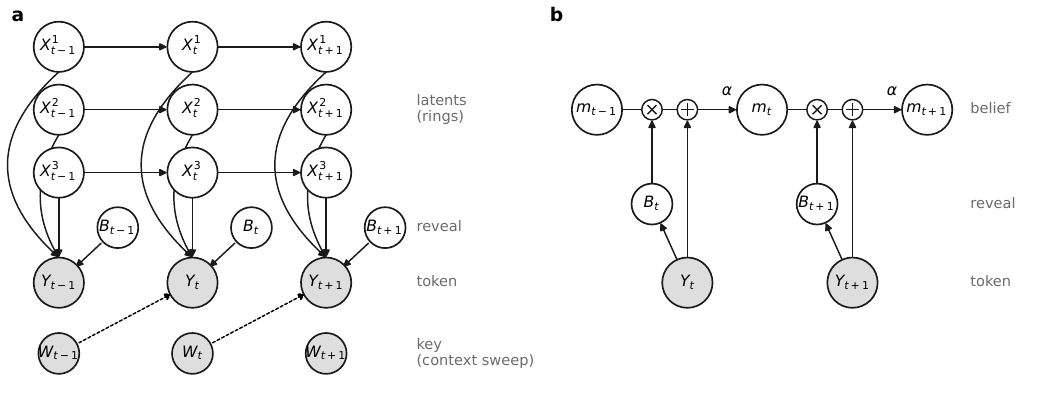}
    \caption{Graphical models of the task \textbf{(a)} and the Bayes filter \textbf{(b)}. \textbf{(a)} Three latents $X^{(i)}_t$ evolve independently on their rings. The reveal variable $B_t$ sets whether the token $Y_t$ shows the latents, and in the context sweep the key $W_{t-1}$ sets whether the shown state is flipped. \textbf{(b)} The belief $m_t$ is the previous belief $m_{t-1}$ propagated by $\alpha$, or, when $B_t=1$, the revealed state from $Y_t$ propagated by $\alpha$.}
    \label{fig:task schematic}
\end{figure}

\subsection{Keyed Filters}
Keyed states are observed differently, but the key state $W_{t-1}$ is directly observable at every token and is invertible. This means that the Fourier mode of the evidence is directly related to the state:
$\phi_j(Y_t) = W_{t-1}\phi_j(X_t)$ (for odd $j$, which includes the only mode when $n=2$). The Bayes filter can be similarly extended, as
\begin{equation}
    m^{(j)}_{t} = \alpha_j \big[(1-B_t)m^{(j)}_{t-1} + B_t \mathbf{1}(W_{t-1} = 1)\phi_j(Y_t) + B_t \mathbf{1}(W_{t-1} = -1)W_{t-1}\phi_j(Y_t)\big]
\end{equation}

The first of the latter two terms accounts for the setting in which the key does not change the state, and the second accounts for the changed state by inverting the rotation to recover the state. This forces the predictive model to change where it reads in token information and writes to the predicted next token, thereby modulating the contextual coherence of the dynamics.

\subsection{The Jacobian Lens}\label{sec: jlens_factors}

In this section we derive the three factors of the J-lens for the Bayes filter (magnitude $a^{(j)}$, lag coherence $\rho_\tau^{(j)}$ and context coherence $\rho_c^{(j)}$) and show how the task parameters set them. We focus on one latent and one Fourier mode $j$: the Jacobian of mode $j$ of the belief at $t+\tau$ with respect to mode $j$ of the evidence at the revealed source $t$, in context $c$,
$$K_j(c,\tau)=\frac{\partial m^{(j)}_{t+\tau}}{\partial\phi_j(Y_t)}.$$
For unkeyed latents $Y_t=X_t$.

This object captures all of the dynamics, as the latents are independent and the filter acts on each Fourier mode separately. Thus the Jacobians between different latents or different modes are zero, and the lens is block diagonal, making each mode its own separate singular vector.
The piecewise-linear components of the dynamics combine to form
\begin{equation}
    K_j(c,\tau) = \alpha_j^{\tau+1} \prod_{k=1}^\tau (1 -B_{t+k}) = \alpha_j^{\tau+1} A_\tau,
\end{equation}
where the product is one if no evidence has been observed across all tokens from position $t+1$ to $t+\tau$, and zero otherwise.

Let $r_j:=\abs{\alpha_j}(1-\varepsilon)$ be the memory of mode $j$ across one step, where $\varepsilon=P(B_t=1)$, and let $N(c)\geq1$ be the number of tokens until the next reveal, so that $P(N=k)=(1-\varepsilon)^{k-1}\varepsilon$. Summing over lags gives $\sum_{\tau=1}^{H}K_j(c,\tau)=\alpha_j S_{j,N(c)}$, where $S_{j,N}:=\sum_{\tau=1}^{\min(N-1,H)}\alpha_j^\tau$ is the response accumulated before the next reveal. This sum matches the lens of Eq.~\eqref{eq:jlens} up to the averaging factor $1/H$, which scales every singular value equally and so leaves their order unchanged. For large $H$ the three coherence parameters are (Appendix~\ref{app: bayes-factors})
\begin{equation}
    \begin{aligned}
        a^{(j)} &= \frac{|\alpha_j|r_j}{1- r_j}, \\
        \rho_\tau^{(j)} &= \frac{(1-r_j)\,\E_c\abs{S_{j,N}}}{r_j},\\
        \rho_c^{(j)}&=|\psi|\frac{r_j}{\abs{1-r_je^{\mathrm{i}\theta_j}}\,\E_c\abs{S_{j,N}}},
    \end{aligned}
    \label{eq:bayes-factors}
\end{equation}
where $\E_c\abs{S_{j,N}}$ is the series given in Eq.~\eqref{eq:ES_abs_inf}, with $\psi=1$ for unkeyed latents. 
When $\theta_j=0$, every term of $S_{j,N}$ is real and positive, so $\rho_\tau^{(j)}=1, \rho_c^{(j)}=|\psi|$ and the persistence $\lambda$ changes only the magnitude $a^{(j)}$. A nonzero angular velocity $\theta$ rotates successive terms of $S_{j,N}$ so that they partly cancel, which lowers the lag coherence $\rho_\tau^{(j)}$. 

\subsection{Connection to the Transformer}

The Jacobian lens on the transformer is defined by $\bm W_U\bm J_H$, where $\bm W_U$ is the unembedding matrix. Given that the activity after the unembedding layer maps to tokens via a softmax, the J-lens on the transformer studies the effect of perturbations on the log probabilities, rather than the probabilities of each state as in the Bayes filter. The difference on the Jacobian is a factor of $\partial \log p(X_t) = \partial p(X_t) / p(X_t)$, a positive reweighting. This reweighting changes the scale of each factor, but away from latents with similar lens weights it should not change their order. We therefore use the theory to predict only the ordering of latents in the transformer, and test it empirically. 

\section{Numerical Results on the Toy Model}

In the following three sections, we first introduce the parameter that is modulated across multiple independent latent states in the HMM (three, across all sections) to induce selective changes in either magnitude $a$, lag coherence $\rho_\tau$, or context coherence $\rho_c$. We then apply the Bayes filter to validate the analytical theory. Next, we train a two-layer decoder-only transformer \citep{vaswani2017attention} with residual dimension 128 and 4 heads per layer to predict the next token given the history (see Appendix \ref{app: transformer} for further details). Our primary result is a validation that the Jacobian lens identifies, \textit{without supervision}, independent modes associated with the latent states, ordered by the selectively modulated coherence factor. To do so,
we fit an encoding model that maps from the Fourier modes (two per latent state for $\sin$ and $\cos$) of a latent to the residual stream, and compute the coherence parameters for these directions. We also 
compute the total cosine similarity between this encoding subspace and the Jacobian lens singular vectors $\bm v_i$ (input side) and $\bm u_i$ (output side) (Appendices~\ref{methods: computing alignment} and~\ref{app: alignment}). We also repeat this process on decoding vectors. Finally, we show that we can predict the persistence $\lambda_0^*$ at which two latents swap places, given different angular velocities $\theta$. 

\subsection{Selectively modulating magnitude}

\begin{figure}[t]
    \centering
    \includegraphics[width=0.9\linewidth]{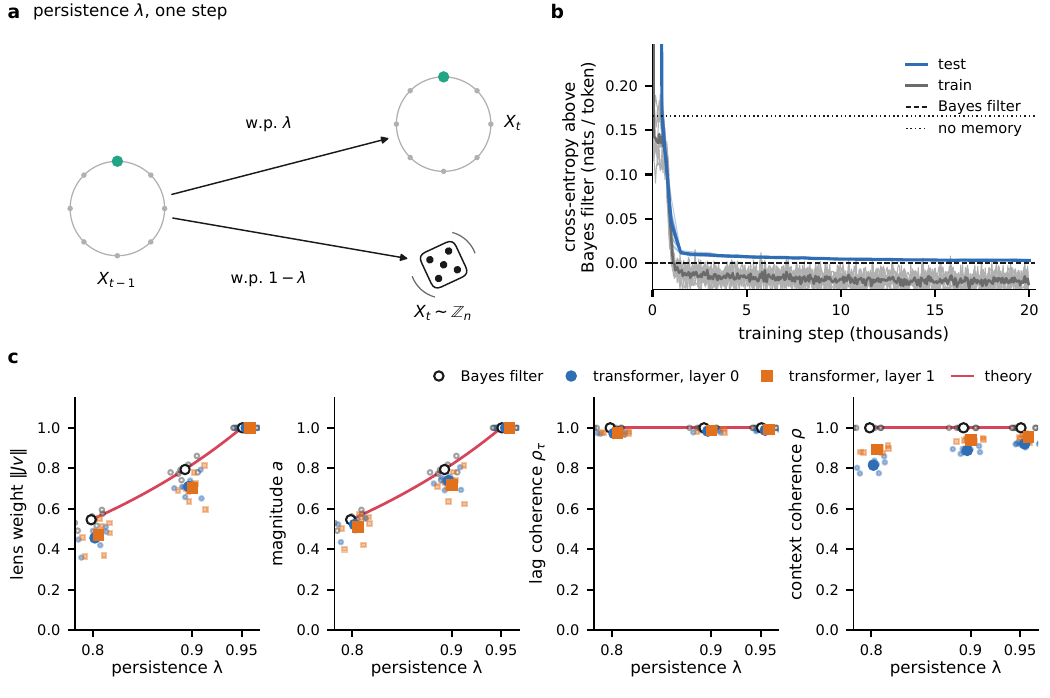}
    \caption{\textbf{The J-lens captures the selective modulation of magnitude $a$ in the Bayes filter and transformers.} \textbf{(a)} Cartoon of the task feature that is modulated: a state $X_{t-1}$ on the ring is kept with probability $\lambda$ (persistence), and sampled uniformly otherwise. \textbf{(b)} Learning curves for the transformer, in comparison to the optimal Bayes filter. 
    \textbf{(c)} J-lens factors of each latent as a function of persistence $\lambda$: lens weight (the singular value $a\,\rho_\tau\,\rho_c$) and magnitude $a$, both relative to their largest value, then lag coherence $\rho_\tau$ and context coherence $\rho_c$. Open circles: Bayes filter. Blue circles and orange squares: transformer layers 0 and 1. Large markers show the mean across seeds and small markers show single seeds. Red line: theory (Eq.~\eqref{eq:bayes-factors}). 
    }
\label{fig: magnitude}
\end{figure}

We first modulate the persistence $\lambda$ by designing a hidden Markov model (HMM) with three latent states, each on a ring, with $\lambda = 0.8, 0.9, 0.95$ respectively. We set the latent state velocity $v = 0$, which holds $\rho_\tau$ and $\rho_c$ fixed while selectively modulating the magnitude $a$ of each latent (Fig. \ref{fig: magnitude}a). All models learn to predict the next token as well as the optimal Bayes filter (Fig. \ref{fig: magnitude}b). On the Bayes filter, the measured J-lens factors match the theory: the lens weight $\norm{\bm J_H\bm v}_2$ tracks the magnitude $a$, while $\rho_\tau$ and $\rho_c$ stay fixed (Fig. \ref{fig: magnitude}c). The J-lens applied to the transformer has similar factors at each layer. As predicted, the J-lens directions identify latents in order of their magnitude $a$ (Fig. \ref{fig: lens_magnitude}).
Across all seeds, this is true for both the input direction $\bm v_i$ (the right singular vector of the J-lens, which reads in the residual stream at layer $\ell$) and the output direction $\bm u_i$ (the left singular vector, which writes to the output). The rank order predicted by the theory is validated on the Bayes filter and transformer (Figs.~\ref{fig: filter_validation} and~\ref{fig: transformer_validation}). The results are also consistent when using decoding vectors for each latent, rather than encoding vectors (Fig. \ref{fig: decoding_magnitude}), albeit with slightly smaller alignment values.

\subsection{Selectively modulating lag coherence}

\begin{figure}[t]
    \centering
    \includegraphics[width=0.9\linewidth]{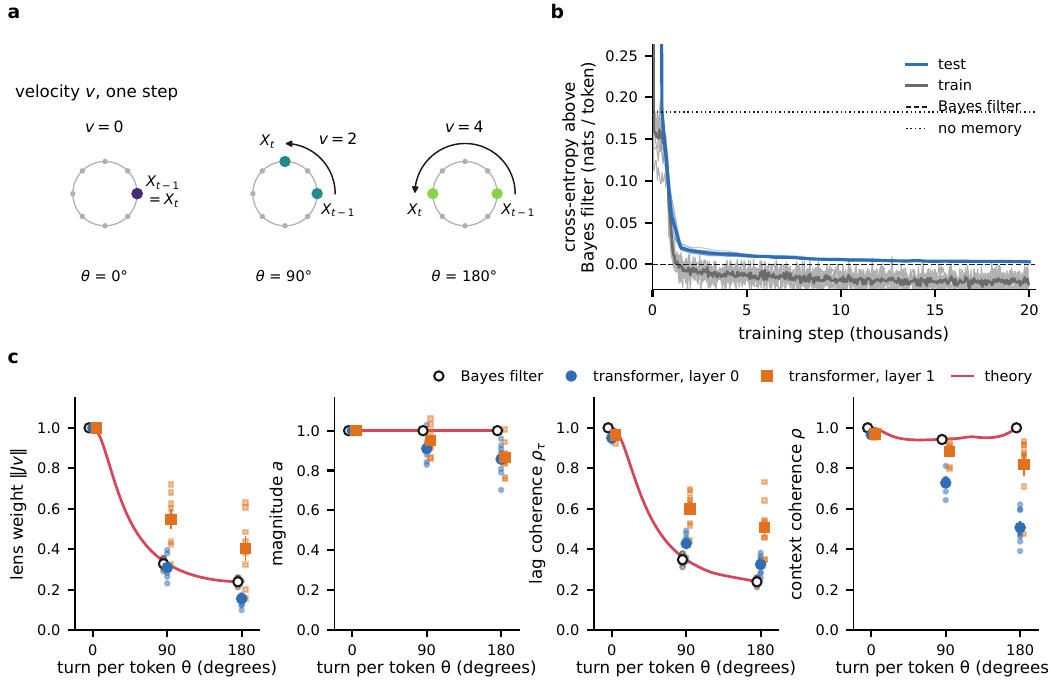}
    \caption{\textbf{The J-lens captures the selective modulation of lag coherence $\rho_\tau$ in the Bayes filter and transformers.}
    \textbf{(a)} The velocity $v$ of the state along the ring, with angular velocity $\theta = 2\pi v/n$, modulates lag coherence $\rho_\tau$.
    \textbf{(b)} Learning curves of the transformer, as in Fig. \ref{fig: magnitude}.
    \textbf{(c)} J-lens factors as a function of angular velocity $\theta$, as in Fig. \ref{fig: magnitude}.}
    \label{fig: lag}
\end{figure}

Next, we selectively modulated the lag coherence $\rho_\tau$ by changing the velocity $v$ of each latent on the ring (Fig. \ref{fig: lag}a). We placed states on a ring of 8 indices, and selected three latents again with velocities $v=0$ (perfectly coherent), $2$, and $4$ (maximally incoherent), that is, $\theta = 0, \pi/2, \pi$. Our theory again predicts that the J-lens should identify these latents in order of their lag coherence $\rho_\tau$. As in the prior section, we found that for transformers trained to optimality (Fig. \ref{fig: lag}b), this played out empirically in every seed at layer 0 and all but one seed at layer 1 (Fig.~\ref{fig: lens_lag}). The lag coherence $\rho_\tau$ of the J-lens applied to the transformer on each latent separately was similarly modulated most by the velocity, while $a$ and $\rho_c$ were less sensitive (Fig. \ref{fig: lag}c). Again, the results are also consistent when using decoding vectors for each latent, rather than encoding vectors (Fig. \ref{fig: decoding_lag}).

\subsubsection{\texorpdfstring{$\omega$}{omega}-lens family} As described in Sec.~\ref{subsec: omega lens}, we also evaluated the $\omega$-J-lenses on the lag coherence sweep for both transformer layers and the Bayes filter. We swept the $\omega$ parameter across the whole range, and found that it preferentially upweighted latent directions with lag coherence parameters (tuned via $\theta$ in the task) that maximally aligned with the $\omega$ (Fig. \ref{fig: omega-lenses}). This suggests that the $\omega$-J-lens family can be used to selectively identify rapidly changing information across tokens (like syntax) or slowly changing information, by sweeping $\omega$.

\begin{figure}
    \centering
    \includegraphics[width=\linewidth]{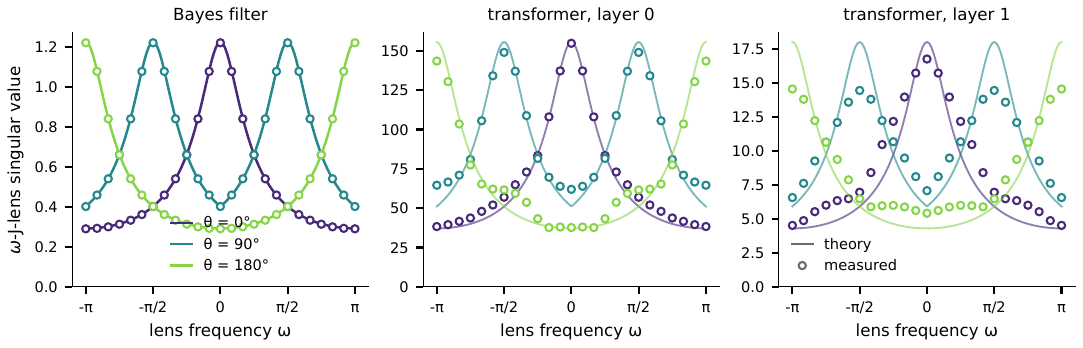}
    \caption{\textbf{The $\omega$-J-lens selectively upweights latent directions with lag coherences aligned with $\omega$.} Solid line indicates the theoretical curves (optimally rescaled from the Bayes filter theory to fit the transformers), points indicate real sampled data. Colors (denoting latent index) are determined by alignment with latent encoding direction. Results are from one seed.}
    \label{fig: omega-lenses}
\end{figure}

\subsection{Selectively modulating context coherence}

Lastly, we selectively modulate context coherence $\rho_c$ by introducing a key state to latent states with $n=2$ states (Sec. \ref{sec: keystate}). In this HMM, an observed state is flipped to the opposite state with probability $(1-\psi)/2$. Holding all other parameters besides $\psi$ fixed, the context coherence can be modulated independently, since $\rho_c=\psi$ for the Bayes filter (Appendix~\ref{app: transformer}). We select $\psi_i = 0,0.75,1$, ranging from maximally incoherent (0) to coherent (1), and train a transformer on this 3-latent HMM. The J-lens factors of the Bayes filter match the theory, with $\rho_c$ tracking $\psi$ (Fig. \ref{fig: context}). The Bayes filter J-lens identifies the latents in the predicted order, and the transformer does as well (Fig. \ref{fig: lens_context}). 

This lens is noisier than the previous analyses, potentially due to the increase in noise in the observed states added by the key variable. In both layers, several J-lens directions also align with patterns other than the latent belief states, mainly the reveal state $B_t$ (whether or not the latent states are observed) of the present token and the key state $W_{t-1}$ (Fig. \ref{fig: other_patterns_context}). This suggests that the J-lens can identify other important features in the latent space in an unsupervised fashion. The lens factors (Fig. \ref{fig: context}c) show that both transformer layers compress the range of context coherence: $\rho_c$ is much higher than in the Bayes filter for the incoherent latent ($\psi=0$) and lower for the coherent ones. The high $\rho_c$ at $\psi=0$ suggests that transformers may have a bias for extracting context-independent information from data. Finally, the results remain consistent when using decoding vectors for each latent, rather than encoding vectors (Fig. \ref{fig: decoding_context}).

\begin{figure}[t]
    \centering
    \includegraphics[width=0.9\linewidth]{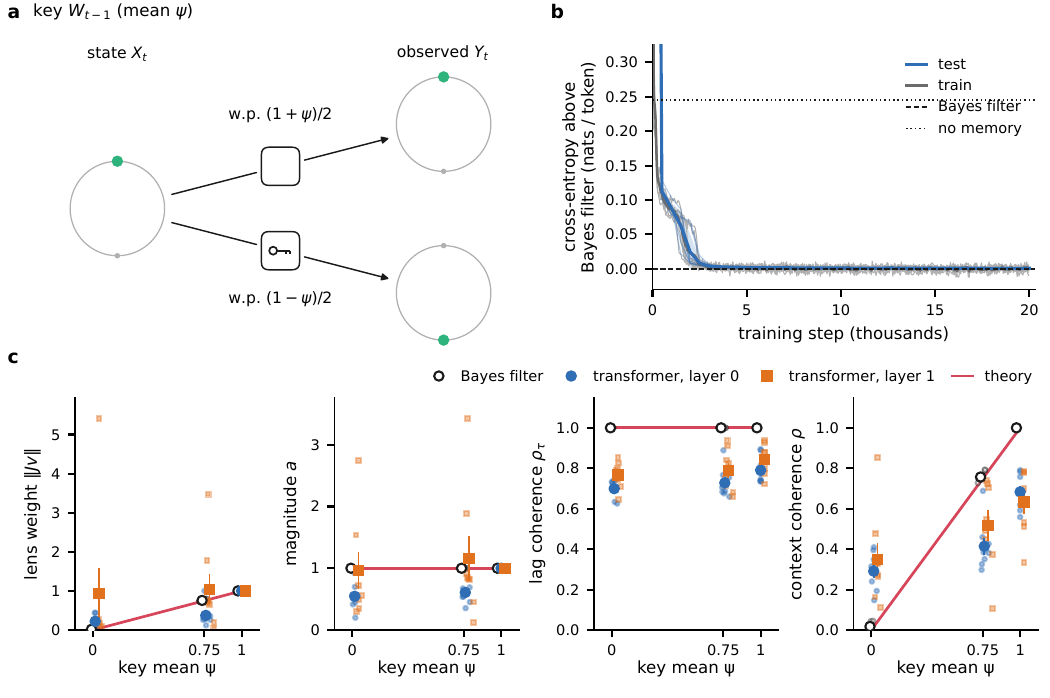}
    \caption{\textbf{The J-lens captures the selective modulation of context coherence $\rho_c$ in the Bayes filter and transformers.}
    \textbf{(a)} For keyed features, the true state is revealed in the token with probability $(1+\psi)/2$. Otherwise, the state is revealed inverted, so that $\rho_c=\psi$. 
    \textbf{(b)} Learning curves of the transformer, as in Fig. \ref{fig: magnitude}.
    \textbf{(c)} J-lens factors as a function of key mean $\psi$, as in Fig. \ref{fig: magnitude}.
    }
    \label{fig: context}
\end{figure}

\begin{figure}[hbpt]
    \centering
    \includegraphics[width=0.9\linewidth]{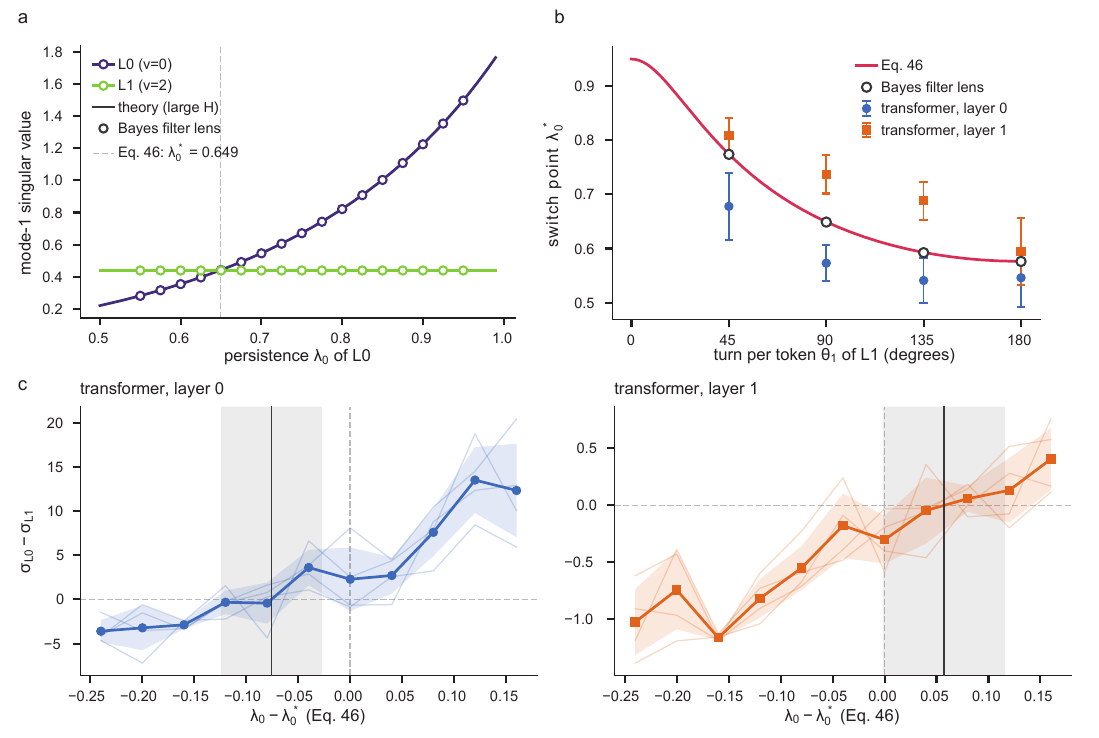}
    \caption{\textbf{Predictive theory of crossing of latents in the J-lens.}
    \textbf{(a)} Empirical singular values $\sigma$ (circles) and analytically derived values in the large-$H$ limit (line) for the J-lens applied to a Bayes filter with two latents (L0, L1). $\lambda_0$ is swept for L0, while $\lambda_1$ is held fixed for L1. Vertical dotted line denotes the predicted crossing point $\lambda_0^*$.
    \textbf{(b)} Predicted crossing point $\lambda_0^*$ (Eq.~\eqref{eq:lambda_star}) and empirical crossing point (circles) as a function of the angular velocity $\theta_1$ of L1.
    \textbf{(c)} Difference in singular values $\sigma_0-\sigma_1$ as a function of $\lambda_0-\lambda_0^*$, for the transformer trained on the same HMM as in (a). Vertical solid line denotes the mean crossing point across seeds, gray shading denotes standard deviation. Each curve is averaged over $4$ variations of $\theta_1$ (circles in (b)). 
    }
    \label{fig: crossing}
\end{figure}

\subsection{Theory predicts the crossing of two latents}

To further emphasize that our theory can predict the ordering of features in the J-lens, we show that our theory can predict when two latents are swapped in the Bayes filter, and that this predicted crossing point roughly translates to the transformer. We simulate an HMM with no keys here for simplicity. For two latents with angular velocities $\theta_0 = 0$ and $\theta_1$, our theory predicts a persistence value $\lambda_0^*$ for a given $\lambda_1$ and $\theta_1$, at which two latents swap order in the J-lens (Appendix~\ref{app: long-window}, Eq.~\eqref{eq:lambda_star}). On the Bayes filter, this theory plays out extremely cleanly (Fig. \ref{fig: crossing}a,b) across a range of $\lambda_1$ and $\theta_1$ values, with the theory exactly predicting the crossing point $\lambda_0^*$ of two latents. Applied to the transformer, our results generalize with a small bias (up to about $0.07$ in $\lambda_0$) in opposite directions at each layer (Fig. \ref{fig: crossing}c). Together, this suggests that even though the Bayes filter theory does not perfectly translate to the transformer, as expected, the predictions are still quite informative.

\section{Discussion}

We have shown that the Jacobian lens, as introduced by \citet{gurnee2026verbalizable}, identifies modes in any temporally and spatially varying dynamical system that jointly maximize three different statistical features:

\paragraph{Magnitude $a$.} Magnitude measures how a state grows or shrinks across a `temporal' dimension. In transformers, it combines growth across layers, with depth exponent $\chi$, and persistence across lags, with rate $\mu$. However, this could be defined differently in other systems. Magnitude aligns most directly with the standard notion of a finite-time Lyapunov exponent, which typically measures growth across time in a dynamical system.

 \paragraph{Lag coherence $\rho_\tau$.} Lag coherence measures how a feature is modulated across `spatial' lags $\tau$. In transformers, this is measured across token positions, where a mode rotates at frequency $\theta$. Cross-token dynamics in language models have previously been studied as multi-particle systems \citep{lu2019understanding,geshkovski2023mathematical}, in which each token (particle) has a particular spatial position and interacts with others via self-attention. The most coherent spatial modes are those with zero spatial frequency ($\theta\approx0$). In Lyapunov theory, coherent directions relate to the measurement of hydrodynamic Lyapunov modes, which are the spatially coherent Lyapunov vectors. 

 \paragraph{Context coherence $\rho_c$.} Context coherence measures how much the dynamics are modulated across different states or time periods. Roughly, maximizing context coherence identifies modes that are \textit{autonomous} ($\rho_c\approx1$), which is to say \textit{independent with respect to the context}. The transformer architecture is highly context-dependent, whereas autonomous dynamical systems are not. That the transformer learned features with much higher $\rho_c$ than the Bayes filter for the context-incoherent latent (Fig.~\ref{fig: context}c) suggests that transformers might learn to modulate contextual information via context-independent dynamics, which could be a fundamental mechanism for symbolic processing. This further suggests that it may be viable to study transformers more generally as dynamical systems \citep{ostrow2024delay} just like recurrent neural networks \citep{sussillo2013opening}. 

Our analysis is consistent with a number of phenomena in the original J-lens paper \citep{gurnee2026verbalizable}. As described previously, high-magnitude, lag-coherent, and context-coherent directions need not have large variance (Corollary~\ref{cor:variance}). This is consistent with the finding of \citet{gurnee2026verbalizable} that the J-space is never more than 10\% of activation variance. Lag coherence preference explains why, when Claude was asked to `concentrate on X while solving Y' \citep{gurnee2026verbalizable}, the J-lens identified features corresponding to X, as these features would be prevalent and coherent across all token positions. A preference for contextually coherent directions explains why averaging identifies concepts that are broadly verbalizable, not specific to a single prompt. 

Our work aligns with other results in mechanistic interpretability \citep{shai2024transformers} that study how toy transformers represent belief state geometry in their latent state and learn to decompose independent factors of the generative process into orthogonal subspaces \citep{shai2026transformers}. However, this work has generally relied on linear decoding of these belief states, which requires \textit{a priori} knowledge of the optimal belief states at each timepoint. Our results here suggest that the J-lens could be a way to identify these belief states in an unsupervised manner, which would be effective for scaling to real-world behaviors. 

This analysis also suggests that a family of Jacobian lenses could be designed to capture different types of information within transformers. For example, local information could be emphasized by averaging over high-frequency Fourier modes, rather than low-frequency, as we showed in Fig. \ref{fig: omega-lenses}. This could be relevant for isolating and understanding fast-changing information, such as those on the level of syntax that change word-by-word. Contextually incoherent information could be isolated by identifying Jacobian modes that minimize context coherence $\rho_c$, rather than maximize it. Given recent work on estimating Jacobians from dynamical data \citep{eisen2026characterizing}, the $\omega$-J-lenses could productively be applied to biological neural data as well. Together, our results clarify the mathematical nature of the J-lens, which is necessary for such a useful and efficient interpretability tool. 

\paragraph{Limitations.} We fit the J-lens on only one dataset and used only one layer and lag window for the open-source LLMs, and it is unclear how changing these factors changes the J-lens. Our analysis additionally relies on a singular value decomposition of the J-lens matrix for analytical tractability, while the original J-space comes from a sparse non-negative decomposition.

\FloatBarrier 

\bibliography{bib}

@article{gurnee2026verbalizable,
  title   = {Verbalizable Representations Form a Global Workspace in Language Models},
  author  = {Gurnee, Wes and Sofroniew, Nicholas and Pearce, Adam and Piotrowski, Mateusz and Kauvar, Isaac and Chen, Runjin and Soligo, Anna and Bogdan, Paul and Ong, Euan and Wang, Rowan and Thompson, Ben and Abrahams, David and Kantamneni, Subhash and Ameisen, Emmanuel and Batson, Joshua and Lindsey, Jack},
  journal = {arXiv preprint arXiv:2607.15495},
  year    = {2026}
}

@inproceedings{
shai2026transformers,
title={Transformers learn factored representations},
author={Adam Shai and Loren Amdahl-Culleton and Casper L. Christensen and Henry R Bigelow and Fernando Rosas and Alexander Boyd and Eric A. Alt and Kyle J Ray and Paul M. Riechers},
booktitle={Forty-third International Conference on Machine Learning},
year={2026},
url={https://openreview.net/forum?id=flDGB9On0h}
}

@article{shai2024transformers,
  title={Transformers represent belief state geometry in their residual stream},
  author={Shai, Adam S and Marzen, Sarah E and Teixeira, Lucas and Oldenziel, Alexander G and Riechers, Paul M},
  journal={Advances in Neural Information Processing Systems},
  volume={37},
  pages={75012--75034},
  year={2024}
}

@book{baars1988cognitive,
  title     = {A Cognitive Theory of Consciousness},
  author    = {Baars, Bernard J.},
  publisher = {Cambridge University Press},
  year      = {1988}
}

@article{dehaene1998neuronal,
  title   = {A neuronal model of a global workspace in effortful cognitive tasks},
  author  = {Dehaene, Stanislas and Kerszberg, Michel and Changeux, Jean-Pierre},
  journal = {Proceedings of the National Academy of Sciences},
  volume  = {95},
  number  = {24},
  pages   = {14529--14534},
  year    = {1998}
}

@article{oseledets1968multiplicative,
  title   = {A multiplicative ergodic theorem. {L}yapunov characteristic numbers for dynamical systems},
  author  = {Oseledets, Valery I.},
  journal = {Transactions of the Moscow Mathematical Society},
  volume  = {19},
  pages   = {197--231},
  year    = {1968}
}

@article{eckmann1985ergodic,
  title   = {Ergodic theory of chaos and strange attractors},
  author  = {Eckmann, Jean-Pierre and Ruelle, David},
  journal = {Reviews of Modern Physics},
  volume  = {57},
  number  = {3},
  pages   = {617--656},
  year    = {1985}
}

@article{deissler1987velocity,
  title   = {Velocity-dependent {L}yapunov exponents as a measure of chaos for open-flow systems},
  author  = {Deissler, Robert J. and Kaneko, Kunihiko},
  journal = {Physics Letters A},
  volume  = {119},
  number  = {8},
  pages   = {397--402},
  year    = {1987}
}

@article{benettin_lyapunov_1980,
	title = {Lyapunov {Characteristic} {Exponents} for smooth dynamical systems and for hamiltonian systems; a method for computing all of them. {Part} 1: {Theory}},
	volume = {15},
	issn = {1572-9648},
	url = {https://doi.org/10.1007/BF02128236},
	doi = {10.1007/BF02128236},
	number = {1},
	journal = {Meccanica},
	author = {Benettin, Giancarlo and Galgani, Luigi and Giorgilli, Antonio and Strelcyn, Jean-Marie},
	month = mar,
	year = {1980},
	pages = {9--20},
}

@article{eisen2026characterizing,
  title={Characterizing control between interacting subsystems with deep Jacobian estimation},
  author={Eisen, Adam J and Ostrow, Mitchell and Chandra, Sarthak and Kozachkov, Leo and Miller, Earl and Fiete, Ila},
  journal={Advances in Neural Information Processing Systems},
  volume={38},
  pages={51361--51406},
  year={2025}
}

@inproceedings{clarkdca,
 author = {Clark, David and Livezey, Jesse and Bouchard, Kristofer},
 booktitle = {Advances in Neural Information Processing Systems},
 editor = {H. Wallach and H. Larochelle and A. Beygelzimer and F. d\textquotesingle Alch\'{e}-Buc and E. Fox and R. Garnett},
 pages = {},
 publisher = {Curran Associates, Inc.},
 title = {Unsupervised Discovery of Temporal Structure in Noisy Data with Dynamical Components Analysis},
 url = {https://proceedings.neurips.cc/paper_files/paper/2019/file/704cddc91e28d1a5517518b2f12bc321-Paper.pdf},
 volume = {32},
 year = {2019}
}

@article{sfapaper,
author = {Wiskott, Laurenz and Sejnowski, Terrence J.},
title = {Slow feature analysis: unsupervised learning of invariances},
year = {2002},
issue_date = {April 2002},
publisher = {MIT Press},
address = {Cambridge, MA, USA},
volume = {14},
number = {4},
issn = {0899-7667},
url = {https://doi.org/10.1162/089976602317318938},
doi = {10.1162/089976602317318938},
journal = {Neural Comput.},
month = apr,
pages = {715--770},
numpages = {56}
}

@article{engelken2024,
  title = {Lyapunov spectra of chaotic recurrent neural networks},
  author = {Engelken, Rainer and Wolf, Fred and Abbott, L. F.},
  journal = {Phys. Rev. Res.},
  volume = {5},
  issue = {4},
  pages = {043044},
  numpages = {28},
  year = {2023},
  month = {Oct},
  publisher = {American Physical Society},
  doi = {10.1103/PhysRevResearch.5.043044},
  url = {https://link.aps.org/doi/10.1103/PhysRevResearch.5.043044}
}

@article{kenfack2013convective,
  title   = {Convective {L}yapunov spectra},
  author  = {Jiotsa, Kenfack and Jiotsa, Aurelien and Politi, Antonio and Torcini, Alessandro},
  journal = {Journal of Physics A: Mathematical and Theoretical},
  volume  = {46},
  number  = {25},
  pages   = {254013},
  year    = {2013},
  doi     = {10.1088/1751-8113/46/25/254013}
}

@article{eckmann2000hydrodynamic,
  title   = {Hydrodynamic {L}yapunov modes in translation-invariant systems},
  author  = {Eckmann, Jean-Pierre and Gat, Omri},
  journal = {Journal of Statistical Physics},
  volume  = {98},
  number  = {3},
  pages   = {775--798},
  year    = {2000}
}

@inproceedings{vaswani2017attention,
  title     = {Attention is all you need},
  author    = {Vaswani, Ashish and Shazeer, Noam and Parmar, Niki and Uszkoreit, Jakob and Jones, Llion and Gomez, Aidan N. and Kaiser, {\L}ukasz and Polosukhin, Illia},
  booktitle = {Advances in Neural Information Processing Systems},
  volume    = {30},
  year      = {2017}
}

@article{lu2019understanding,
  title   = {Understanding and improving {T}ransformer from a multi-particle dynamic system point of view},
  author  = {Lu, Yiping and Li, Zhuohan and He, Di and Sun, Zhiqing and Dong, Bin and Qin, Tao and Wang, Liwei and Liu, Tie-Yan},
  journal = {arXiv preprint arXiv:1906.02762},
  year    = {2019}
}

@article{geshkovski2023mathematical,
  title   = {A mathematical perspective on {T}ransformers},
  author  = {Geshkovski, Borjan and Letrouit, Cyril and Polyanskiy, Yury and Rigollet, Philippe},
  journal = {arXiv preprint arXiv:2312.10794},
  year    = {2023}
}

@article{sussillo2013opening,
  title   = {Opening the black box: low-dimensional dynamics in high-dimensional recurrent neural networks},
  author  = {Sussillo, David and Barak, Omri},
  journal = {Neural Computation},
  volume  = {25},
  number  = {3},
  pages   = {626--649},
  year    = {2013}
}

@article{ostrow2024delay,
  title   = {Delay embedding theory of neural sequence models},
  author  = {Ostrow, Mitchell and Eisen, Adam and Fiete, Ila},
  journal = {arXiv preprint arXiv:2406.11993},
  year    = {2024}
}

@misc{nostalgebraist2020logitlens,
  title        = {interpreting {GPT}: the logit lens},
  author       = {nostalgebraist},
  howpublished = {LessWrong, \url{https://www.lesswrong.com/posts/AcKRB8wDpdaN6v6ru/interpreting-gpt-the-logit-lens}},
  year         = {2020}
}

@article{belrose2023tuned,
  title   = {Eliciting latent predictions from transformers with the tuned lens},
  author  = {Belrose, Nora and Ostrovsky, Igor and McKinney, Lev and Furman, Zach and Smith, Logan and Halawi, Danny and Biderman, Stella and Steinhardt, Jacob},
  journal = {arXiv preprint arXiv:2303.08112},
  year    = {2023}
}

@inproceedings{wendler2024llamas,
  title     = {Do Llamas Work in {E}nglish? On the Latent Language of Multilingual Transformers},
  author    = {Wendler, Chris and Veselovsky, Veniamin and Monea, Giovanni and West, Robert},
  booktitle = {Proceedings of the 62nd Annual Meeting of the Association for Computational Linguistics (ACL)},
  year      = {2024}
}

@inproceedings{halawi2024overthinking,
  title     = {Overthinking the Truth: Understanding how Language Models Process False Demonstrations},
  author    = {Halawi, Danny and Denain, Jean-Stanislas and Steinhardt, Jacob},
  booktitle = {International Conference on Learning Representations},
  year      = {2024}
}

@inproceedings{pal2023future,
  title     = {Future Lens: Anticipating Subsequent Tokens from a Single Hidden State},
  author    = {Pal, Koyena and Sun, Jiuding and Yuan, Andrew and Wallace, Byron C. and Bau, David},
  booktitle = {Proceedings of the 27th Conference on Computational Natural Language Learning (CoNLL)},
  year      = {2023},
  doi       = {10.18653/v1/2023.conll-1.37}
}

@article{yan2026short,
  title   = {Short Horizons and Sparse Concepts: a Mathematical View of the Readout in the {J}-lens},
  author  = {Yan, Shi-Qi and Ding, Kai-Xuan and Tan, Chao-Hong and Chen, Qian and Wang, Wen and Li, Xiangang and Ling, Zhen-Hua},
  journal = {arXiv preprint arXiv:2608.25347},
  year    = {2026}
}

@article{stein1981estimation,
  title   = {Estimation of the mean of a multivariate normal distribution},
  author  = {Stein, Charles M.},
  journal = {The Annals of Statistics},
  volume  = {9},
  number  = {6},
  pages   = {1135--1151},
  year    = {1981}
}

@inproceedings{dettmers2022llmint8,
  title     = {{LLM.int8()}: 8-bit Matrix Multiplication for Transformers at Scale},
  author    = {Dettmers, Tim and Lewis, Mike and Belkada, Younes and Zettlemoyer, Luke},
  booktitle = {Advances in Neural Information Processing Systems},
  volume    = {35},
  year      = {2022}
}

@inproceedings{sun2024massive,
  title     = {Massive Activations in Large Language Models},
  author    = {Sun, Mingjie and Chen, Xinlei and Kolter, J. Zico and Liu, Zhuang},
  booktitle = {Conference on Language Modeling (COLM)},
  year      = {2024}
}

@article{su2024roformer,
  title   = {{RoFormer}: Enhanced transformer with rotary position embedding},
  author  = {Su, Jianlin and Ahmed, Murtadha and Lu, Yu and Pan, Shengfeng and Bo, Wen and Liu, Yunfeng},
  journal = {Neurocomputing},
  volume  = {568},
  pages   = {127063},
  year    = {2024}
}

@inproceedings{loshchilov2019decoupled,
  title     = {Decoupled Weight Decay Regularization},
  author    = {Loshchilov, Ilya and Hutter, Frank},
  booktitle = {International Conference on Learning Representations},
  year      = {2019}
}

@misc{wikimedia2023dump,
  author       = {{Wikimedia Foundation}},
  title        = {Wikimedia Downloads: Wikipedia dump of 2023-11-01},
  year         = {2023},
  howpublished = {\url{https://dumps.wikimedia.org}}
}

@book{tunstall2022nlp,
  title     = {Natural Language Processing with Transformers},
  author    = {Tunstall, Lewis and von Werra, Leandro and Wolf, Thomas},
  publisher = {O'Reilly Media},
  year      = {2022}
}

@inproceedings{lu2021codexglue,
  title     = {{CodeXGLUE}: A Machine Learning Benchmark Dataset for Code Understanding and Generation},
  author    = {Lu, Shuai and Guo, Daya and Ren, Shuo and Huang, Junjie and Svyatkovskiy, Alexey and Blanco, Ambrosio and Clement, Colin and Drain, Dawn and Jiang, Daxin and Tang, Duyu and Li, Ge and Zhou, Lidong and Shou, Linjun and Zhou, Long and Tufano, Michele and Gong, Ming and Zhou, Ming and Duan, Nan and Sundaresan, Neel and Deng, Shao Kun and Fu, Shengyu and Liu, Shujie},
  booktitle = {Proceedings of the Neural Information Processing Systems Track on Datasets and Benchmarks},
  year      = {2021}
}

@misc{meyer2024synthetic,
  title        = {Synthetic-Text-To-{SQL}: A synthetic dataset for training language models to generate {SQL} queries from natural language prompts},
  author       = {Meyer, Yev and Emadi, Marjan and Nathawani, Dhruv and Ramaswamy, Lipika and Boyd, Kendrick and Van Segbroeck, Maarten and Grossman, Matthew and Mlocek, Piotr and Newberry, Drew},
  year         = {2024},
  howpublished = {\url{https://huggingface.co/datasets/gretelai/synthetic_text_to_sql}}
}

@inproceedings{paster2024openwebmath,
  title     = {{OpenWebMath}: An Open Dataset of High-Quality Mathematical Web Text},
  author    = {Paster, Keiran and Dos Santos, Marco and Azerbayev, Zhangir and Ba, Jimmy},
  booktitle = {International Conference on Learning Representations},
  year      = {2024}
}

@article{cobbe2021training,
  title   = {Training Verifiers to Solve Math Word Problems},
  author  = {Cobbe, Karl and Kosaraju, Vineet and Bavarian, Mohammad and Chen, Mark and Jun, Heewoo and Kaiser, Lukasz and Plappert, Matthias and Tworek, Jerry and Hilton, Jacob and Nakano, Reiichiro and Hesse, Christopher and Schulman, John},
  journal = {arXiv preprint arXiv:2110.14168},
  year    = {2021}
}

@inproceedings{chen2021dialogsum,
  title     = {{DialogSum}: A Real-Life Scenario Dialogue Summarization Dataset},
  author    = {Chen, Yulong and Liu, Yang and Chen, Liang and Zhang, Yue},
  booktitle = {Findings of the Association for Computational Linguistics: ACL-IJCNLP 2021},
  pages     = {5062--5074},
  year      = {2021}
}

@inproceedings{hermann2015teaching,
  title     = {Teaching Machines to Read and Comprehend},
  author    = {Hermann, Karl Moritz and Ko{\v{c}}isk{\'y}, Tom{\'a}{\v{s}} and Grefenstette, Edward and Espeholt, Lasse and Kay, Will and Suleyman, Mustafa and Blunsom, Phil},
  booktitle = {Advances in Neural Information Processing Systems},
  volume    = {28},
  year      = {2015}
}

@inproceedings{see2017get,
  title     = {Get To The Point: Summarization with Pointer-Generator Networks},
  author    = {See, Abigail and Liu, Peter J. and Manning, Christopher D.},
  booktitle = {Proceedings of the 55th Annual Meeting of the Association for Computational Linguistics},
  pages     = {1073--1083},
  year      = {2017}
}

@inproceedings{tuggener2020ledgar,
  title     = {{LEDGAR}: A Large-Scale Multi-label Corpus for Text Classification of Legal Provisions in Contracts},
  author    = {Tuggener, Don and von D{\"a}niken, Pius and Peetz, Thomas and Cieliebak, Mark},
  booktitle = {Proceedings of the Twelfth Language Resources and Evaluation Conference},
  pages     = {1235--1241},
  year      = {2020}
}

@inproceedings{chalkidis2022lexglue,
  title     = {{LexGLUE}: A Benchmark Dataset for Legal Language Understanding in {E}nglish},
  author    = {Chalkidis, Ilias and Jana, Abhik and Hartung, Dirk and Bommarito, Michael and Androutsopoulos, Ion and Katz, Daniel Martin and Aletras, Nikolaos},
  booktitle = {Proceedings of the 60th Annual Meeting of the Association for Computational Linguistics},
  pages     = {4310--4330},
  year      = {2022}
}

@article{wedin1972perturbation,
  title   = {Perturbation bounds in connection with singular value decomposition},
  author  = {Wedin, Per-{\AA}ke},
  journal = {BIT Numerical Mathematics},
  volume  = {12},
  number  = {1},
  pages   = {99--111},
  year    = {1972}
}
\bibliographystyle{iclr2027_conference}

\appendix

\section{Jacobian singular values under averaging}\label{app: avg_jac}
 
Fix $U,V$ and singular values $s_j = e^{-\beta j}$ for $j=1,\dots,d$, with
$S=\mathrm{diag}(s_1,\dots,s_d)$. For each sample $i=1,\dots,N$, draw a random
permutation $P_i$ and set
\begin{equation}
  X_i \;=\; U P_i S P_i^\top V^\top .
\end{equation}
Every $X_i$ has the same singular values $e^{-\beta}, e^{-2\beta},\dots,e^{-\beta d}$.
Since $U,V$ are shared,
\begin{equation}
  \bar X \;=\; \frac1N\sum_i X_i \;=\; U\Big(\frac1N\sum_i P_i S P_i^\top\Big)V^\top
  \;=\; U\bar S V^\top,
\end{equation}
so the singular values of $\bar X$ are the diagonal entries of $\bar S$,
\begin{equation}
  \bar S_{kk} \;=\; \frac1N\sum_i s_{\pi_i(k)} .
\end{equation}
Each $s_{\pi_i(k)}$ is a uniform draw from $\{s_1,\dots,s_d\}$, so
\begin{equation}
  \E\,\bar S_{kk} \;=\; \bar s \;=\; \frac1d\sum_{j=1}^{d} e^{-\beta j},
  \qquad
  \mathrm{Var}\,\bar S_{kk} \;=\; \frac{\mathrm{Var}(s)}{N}\;\xrightarrow{N\to\infty}\;0 .
\end{equation}
Every singular value of $\bar X$ therefore converges to the same value $\bar s$.
Each $X_i$ has the decaying spectrum $e^{-\beta k}$, but the mean $\bar X$ has a
flat spectrum: $\bar X \to \bar s\, U V^\top$.

\section{Real Jacobians of LLMs}\label{app: real_jacobians}

For each model (Gemma-2-2B, Llama-3.2-3B, Qwen3.5-9B, Qwen3-14B), we computed per-token Jacobians of the residual stream on 96 text prompts drawn
from 16 public datasets, six prompts each, all from Hugging Face:
\begin{itemize}
  \item Wikipedia, 2023-11-01 dump, in English, Russian, Arabic, Chinese, Japanese and Swahili
        (\texttt{wikimedia/wikipedia}; \citealp{wikimedia2023dump});
  \item Python code from CodeParrot-clean (\texttt{codeparrot/codeparrot-clean}; \citealp{tunstall2022nlp});
  \item JavaScript functions from the CodeXGLUE code-to-text benchmark
        (\texttt{google/code\_x\_glue\_ct\_code\_to\_text}; \citealp{lu2021codexglue});
  \item SQL schemas with queries from Gretel's synthetic text-to-SQL dataset
        (\texttt{gretelai/synthetic\_text\_to\_sql}; \citealp{meyer2024synthetic});
  \item mathematical web text from OpenWebMath (\texttt{open-web-math/open-web-math}; \citealp{paster2024openwebmath});
  \item grade-school math word problems with worked solutions from GSM8K
        (\texttt{openai/gsm8k}; \citealp{cobbe2021training});
  \item dialogues from DialogSum (\texttt{knkarthick/dialogsum}; \citealp{chen2021dialogsum});
  \item news articles from CNN/DailyMail, version 3.0.0
        (\texttt{abisee/cnn\_dailymail}; \citealp{hermann2015teaching,see2017get});
  \item poems (\texttt{merve/poetry})\footnote{\url{https://huggingface.co/datasets/merve/poetry}};
  \item contract clauses from LEDGAR, via LexGLUE
(\texttt{coastalcph/lex\_glue}; \citealp{tuggener2020ledgar,chalkidis2022lexglue});
  \item recipes (\texttt{corbt/all-recipes})\footnote{\url{https://huggingface.co/datasets/corbt/all-recipes}}.
\end{itemize}
Prompts were drawn at random, with a fixed seed, from the first 200 training examples of each
dataset. Only texts of at least 400 characters were kept, and each was truncated to its first
57 tokens. We computed the exact Jacobian from the middle block of all models to the last block, one token ahead:

\begin{equation}
  J_{c,t} \;=\; \frac{\partial\, h_{L-1}^{(c)}(t+1)}{\partial\, h_{\ell}^{(c)}(t)} \;\in\; \mathbb{R}^{d \times d}.
\end{equation}
Jacobians were computed by reverse-mode differentiation, with float32 weights and eager attention,
giving $N = 96 \times 8 = 768$ matrices per model. The mean Jacobian is their element-wise average,
\begin{equation}
  \bar{J} \;=\; \frac{1}{N} \sum_{c,t} J_{c,t}.
\end{equation}
where $c$ denotes context and $t$ denotes position in the text.
To visualize the distribution of these matrices, we computed each Jacobian's top 100 singular values, suppressing the rest as they may be noisy. We rescaled the singular value distribution by taking a logarithm, so that lower values can contribute in dimensionality reduction, and normalized the distribution for all Jacobians. We then fit and applied PCA to the individual Jacobians, then projected the mean Jacobian into this space.  

\begin{figure}
    \centering
    \includegraphics[width=0.5\linewidth]{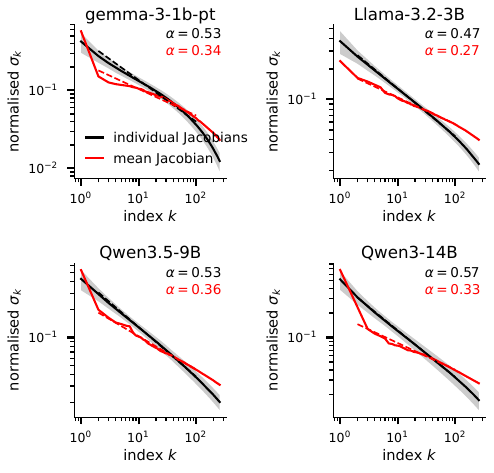}
    \caption{\textbf{Jacobian singular value spectra of four LLMs, with estimated power laws.} Individual Jacobian curve is averaged over all points in Fig.~\ref{fig: real_jac}, with shading denoting standard deviation.}
    \label{fig: jac_spectra}
\end{figure}

\begin{figure}
    \centering
    \includegraphics[width=0.8\linewidth]{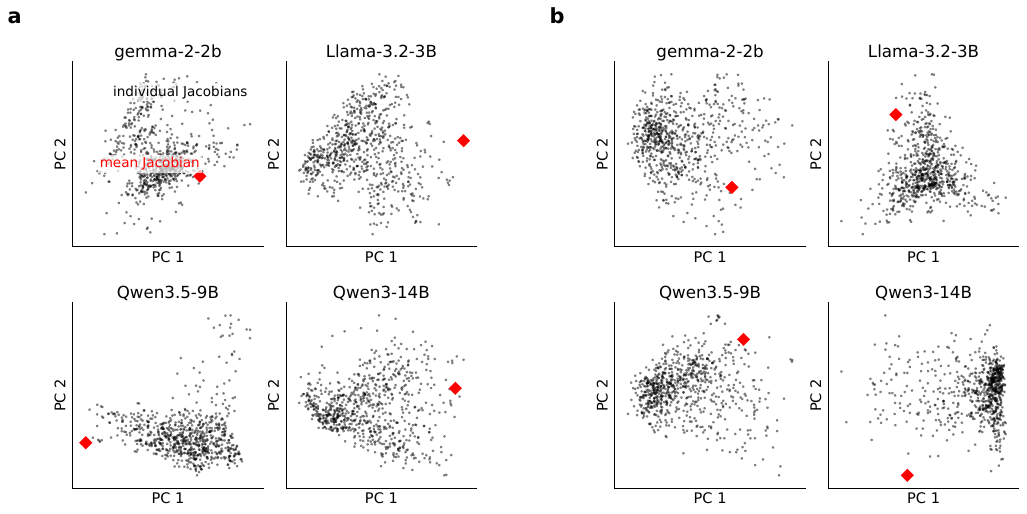}
    \caption{\textbf{Low-dimensional projection of the left and right singular vectors of Jacobians.} The top 100 left and right singular vectors, ranked by singular values as in Fig.~\ref{fig: real_jac}, are projected into two dimensions by computing the subspace distance on Grassmannians $||P_i - P_j||_F = \sqrt{2\sum_k \sin^2 \theta_k}$, where $\theta_k$ are the principal subspace angles and $P_i$ is the projector matrix defined by the singular vectors, $P_i = U_iU_i^T$. This pairwise distance matrix is reduced to two dimensions and visualized with PCA. (a) Left singular vectors. (b) Right singular vectors.}
    \label{appfig: real_jac_svecs}  
\end{figure}

\section{Computing coherence factors for LLM J-lenses}\label{app:jlens_llm}

We compute the J-lens of each LLM in Fig.~\ref{fig: real_jac} from the middle block to the last block
(indexing blocks from zero):
gemma-2-2b ($d = 2304$, $26$ blocks, block $13 \to 25$),
Llama-3.2-3B ($d = 3072$, $28$ blocks, $14 \to 27$),
Qwen3.5-9B ($d = 4096$, $32$ blocks, $16 \to 31$) and
Qwen3-14B ($d = 5120$, $40$ blocks, $20 \to 39$).
Weights are in float32, and matrix products use TF32.

We fit the J-lenses to \nfit{} articles of the WikiText-103 training split (the first $192$ tokens, skipping the first $16$ as they hold attention-sink activations), and evaluate the coherence parameters on English natural-language text
from other datasets: news articles (CNN/DailyMail), dialogues (DialogSum), poems, contract
clauses (LEDGAR) and recipes. We use \nperdataset{} prompts per dataset, \nheldout{} held-out contexts in each individual run.  Lags run over $\tau = 1, \dots, H$ with $H = \horizon$.
We repeat the analysis on independently sampled data, with \nrunssmall{} runs for gemma-2-2b and Llama-3.2-3B and \nrunslarge{} runs for
Qwen3.5-9B and Qwen3-14B. 

The J-lens directions are the top $\ndirs$ right singular vectors $\bm v_i$ of $\bm J_H$, ordered by singular value.
For comparison, we include principal component analysis (PCA), SFA and DCA. 
PCA directions are the top \ndirs{} eigenvectors of the covariance of the source-layer activations $\bm h_{\ell, t}$ at
positions $16$ and later in the \nfit{} fit prompts, ordered by variance. Random directions are \ndirs{} independent
standard Gaussian vectors. All directions are normalized to unit length before passing through the Jacobian. 

SFA and DCA take into account the order of the tokens. Both use the source activations of the fit prompts at positions
$16$ and later, reduced to their top \nreduce{} principal components and whitened to the identity covariance. Let
$\bm z_t$ be the whitened activation at token $t$.
Slow feature analysis takes the \ndirs{} unit vectors $\bm w$ along which $\bm w^\top \bm z_t$
changes least from one token to the next, that is, the eigenvectors of $\E (\bm z_{t+1} - \bm z_t)(\bm z_{t+1} - \bm z_t)^\top$ with
the smallest eigenvalues, the slowest first. Dynamical components analysis takes the \ndirs{}-dimensional subspace whose
trajectory has the largest predictive information: the mutual information between its past and its next $T = \dcawindow$
tokens under a Gaussian model, $\log |\bm C_T| - \tfrac12 \log |\bm C_{2T}|$ with $\bm C_n$ the covariance of $n$ consecutive tokens.
We optimize it by gradient ascent from two starting points (the slow features and a random subspace) and keep the
result with greater predictive information. A subspace has no order, so we order via a quasi Gram-Schmidt process by picking a direction that maximizes the predictive information of its own timeseries, then projecting this direction out of the subspace. Each method (J-lens, PCA, SFA, DCA) thus gives \ndirs{} rank-ordered unit directions, which we compare with the \ndirs{} random directions. For each unit direction $\bm v$ we compute the Jacobian-vector
products $\bm K(c, \tau) \bm v$ for all $H$ lags, and from them the three coherence parameters as in Eq.~\eqref{eq:coherence}:
\begin{align}
  a(\bm v) &= \E_{c, \tau} \norm{\bm K(c, \tau) \bm v}, &
  \rho_\tau(\bm v) &= \frac{\E_c \norm{\E_\tau \bm K(c, \tau) \bm v}}{\E_{c, \tau} \norm{\bm K(c, \tau) \bm v}}, &
  \rho_c(\bm v) &= \frac{\norm{\E_{c, \tau} \bm K(c, \tau) \bm v}}{\E_c \norm{\E_\tau \bm K(c, \tau) \bm v}}.
\end{align}

\begin{figure}
    \centering
    \includegraphics[width=\linewidth]{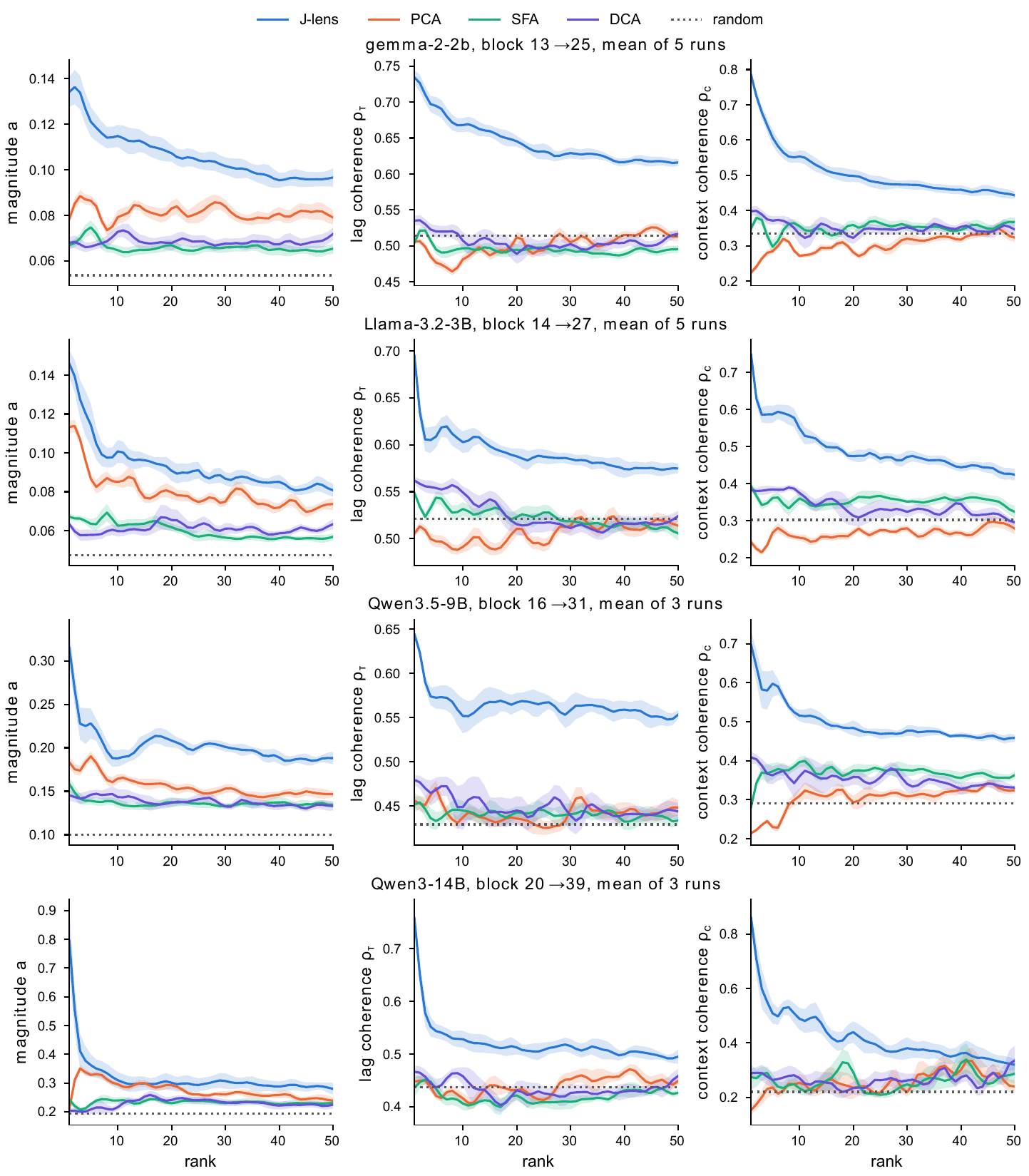}
    \caption{Individual coherence parameter curves for each LLM, with baselines computed separately for each model.}
    \label{fig: coherence_single_llms}
\end{figure}

\section{The shared-mode model}\label{app:mode-model}
This appendix provides the averaging steps and assumptions behind the shared-mode model of Eq.~\eqref{eq:model}.

\paragraph{Averaging.} The shared component has singular strengths
$e^{T\chi_i+\tau\mu_i}$.
Suppressing $\ell$ in the notation, averaging gives
\begin{equation}
 \begin{aligned}
 \bm J_H&=\bm S_H+\E_c\bm R_c,
 &\bm S_H&=\sum_i g_i \bm u_i \bm v_i^*,\\
 \bm R_c&=\frac1H\sum_{\tau=1}^{H}\bm R(c,\ell,\tau),
 &g_i&=e^{T\chi_i}f_H(\gamma_i),\\
 f_H(\gamma)&=\frac1H\sum_{\tau=1}^{H}e^{\tau\gamma}
 =\begin{cases}
 e^{\gamma}(1-e^{H\gamma})/(H(1-e^{\gamma})),&e^{\gamma}\ne1,\\
 1,&e^{\gamma}=1.
 \end{cases}
 \end{aligned}
 \label{eq:model-average}
\end{equation}
The shared singular strengths after averaging are $|g_i|$.
Assume bounded accumulated depth growth and non-growing lag dynamics:
\begin{equation}
 \begin{gathered}
 T\chi_i\leq\Gamma,\qquad \mu_i\leq0,\\
\norm{\bm S_H}_2 = M\leq e^\Gamma,\qquad
 \E_c\norm{\bm R_c}_2^2\leq M^2.
 \end{gathered}
 \label{eq:scale}
\end{equation}
These assumptions prevent the context-specific response from having
arbitrarily greater magnitude than the shared response scale.
Spectral dominance additionally requires shared strengths above the
averaged remainder, as shown in Eq.~\eqref{eq:suppression}.

\paragraph{Overlap of context-specific responses.} For a unit input perturbation $\bm v$ and independent context draws $c,c'$,
\begin{equation}
 \begin{aligned}
 \norm{(\E_c\bm R_c)\bm v}_2^2
 =\bm v^*\E_{c,c'}[\bm R_c^*\bm R_{c'}]\bm v
 =\E_{c,c'}\langle \bm R_c\bm v,\bm R_{c'}\bm v\rangle.
 \end{aligned}
 \label{eq:directional-growth}
\end{equation}
Here $\langle \bm A,\bm B\rangle_F=\operatorname{tr}(\bm A^*\bm B)$ and $*$ denotes
conjugate transpose. Thus growth is the expected Frobenius overlap of
the Jacobians restricted to the same input direction. The unprojected
overlap $\langle \bm R_c,\bm R_{c'}\rangle_F$ sums over input directions instead.

Write $\bm R_c=\sum_j s^R_{cj}\bm u^R_{cj}(\bm v^R_{cj})^*$, using $s,\bm u,\bm v$ for
singular values, output directions, and input directions. Then
\begin{equation}
 \langle \bm R_c\bm v,\bm R_{c'}\bm v\rangle
 =\sum_{j,k}
 \underbrace{s^R_{cj}s^R_{c'k}}_{\text{magnitude}}
 \underbrace{(\bm v^*\bm v^R_{cj})((\bm v^R_{c'k})^*\bm v)}_{\text{input alignment}}
 \underbrace{((\bm u^R_{cj})^*\bm u^R_{c'k})}_{\text{output alignment}}.
 \label{eq:residual-alignment}
\end{equation}
The bound $\sup_{\norm{\bm v}=1}\E_{c,c'}|\langle \bm R_c\bm v,\bm R_{c'}\bm v\rangle|\leq\eta\,\E_c\norm{\bm R_c}_2^2$ used in Eq.~\eqref{eq:suppression} therefore requires that, on average, the input or output directions of the context-specific modes are poorly aligned across contexts.

\paragraph{Perturbation bounds.} Write $\bm E:=\E_c\bm R_c$, so $\norm{\bm E}_2\leq\sqrt\eta\,M$ by Eq.~\eqref{eq:suppression}. By Weyl's inequality, $\abs{\sigma_i(\bm J_H)-\sigma_i(\bm S_H)}\leq\norm{\bm E}_2$. By Wedin's theorem \citep{wedin1972perturbation}, the sine of the largest angle between the top-$k$ right singular subspaces of $\bm J_H$ and $\bm S_H$ is at most $\norm{\bm E}_2/(\Delta_k-\norm{\bm E}_2)$, where $\Delta_k=\sigma_k(\bm S_H)-\sigma_{k+1}(\bm S_H)>\norm{\bm E}_2$.

\paragraph{Response-scale bound.} Since $\norm{\bm J_H\bm v}=\rho_c(\bm v)\,\rho_\tau(\bm v)\,a(\bm v)$ and $a(\bm v)\leq\E_{c,\tau}\norm{\bm K(c,\ell,\tau)}_2$ for a unit input $\bm v$,
\begin{equation}
 \norm{\bm J_H\bm v}\leq\rho_c(\bm v)\,\rho_\tau(\bm v)\,\E_{c,\tau}\norm{\bm K(c,\ell,\tau)}_2 .
\end{equation}
At a fixed response scale, a direction's contribution to the lens is therefore capped by the product of its two coherences.

\section{Proof of Corollary~\ref{cor:variance}}\label{app:corollary-proof}

\begin{proof} 
Let \(\bm e_1,\bm e_2\) be orthonormal directions, and consider the linear response model
\[
\bm K(c,\ell,\tau)
=s\left(\bm e_1\bm e_1^*+\xi_c(-1)^\tau \bm e_2\bm e_2^*\right),
\qquad s>0,
\]
where \(\xi_c\) takes values \(\pm1\) with equal probability. The first direction is perfectly coherent across contexts and constant across future lags; the second changes sign across contexts and oscillates rapidly across lags. Thus
\[
\bm J_H=\E_c\frac1H\sum_{\tau=1}^{H}\bm K(c,\ell,\tau)
=s\,\bm e_1\bm e_1^*.
\]
The unique nonzero singular value is \(s\), and its input direction is \(\bm e_1\).

Independently, choose centered input activations, independent of \(\xi_c\), with covariance
\[
\bm C=\beta\, \bm e_1\bm e_1^*+\bm e_2\bm e_2^*,
\qquad \beta>0.
\]
These responses are realized by the context-dependent linear maps \(\bm h\mapsto \bm K(c,\ell,\tau)\bm h\). The fraction of activation variance captured by projection onto the selected direction is
\[
\frac{\bm e_1^*\bm C\bm e_1}{\operatorname{tr}\bm C}
=\frac{\beta}{1+\beta}
\longrightarrow0
\quad\text{as }\beta\to0.
\]
Interchanging the two covariance eigenvalues makes this fraction approach \(1\), without changing the Jacobians, their coherence or slowness, or the singular values of \(\bm J_H\).

By \eqref{eq:jlens-token}, every nonzero token-indexed J-lens direction in this example is proportional to \(\bm e_1\). Hence even the full span of the J-lens directions can capture arbitrarily little activation variance. Jacobian averaging favors coherent, slowly varying responses, weighted by their response magnitudes. Those properties do not require the corresponding activations to have large variance.
\end{proof}

\section{Bayes filter derivations} \label{app: bayes}
Any
distribution $p$ on $\mathbb{Z}_n$ is defined by its modes $\E_p\phi_j(X)$:
\begin{equation}
 p(x)=\frac1n\sum_{j=0}^{n-1}\E_p[\phi_j(X)]\,\conj{\phi_j(x)} .
 \label{eq:inverse}
\end{equation}

\begin{lemma}[Expected dynamic]\label{lem:step}
For $j\ne0$,
\begin{equation}
 \E[\phi_j(X_{t+1})\mid X_t]=\alpha_j\,\phi_j(X_t),\qquad
 \alpha_j:=\lambda\,c_j\,e^{\mathrm{i}\theta_j},\qquad
 c_j:=1-2q+2q\cos(2\pi j/n).
 \label{eq:cj}
\end{equation}
\end{lemma}

\begin{proof}
The expectation of a mode is $\E\phi_j=\frac1n\sum_y\phi_j(y)=0$. One update step gives
\begin{align}
\phi_j(X_t+v+D_t)=\phi_j(X_t)\,e^{\mathrm{i}\theta_j}e^{2\pi \mathrm{i}jD_t/n}
\end{align}
and 
\begin{align}
 \E e^{2\pi \mathrm{i}jD_t/n}=1-2q+q\left(e^{2\pi \mathrm{i}j/n}+e^{-2\pi \mathrm{i}j/n}\right) =c_j .
\end{align}
\end{proof}

\subsection{Derivation of the J-lens factors}\label{app: bayes-factors}
We use the notation of Section~\ref{sec: jlens_factors}: $K_j(c,\tau)=\alpha_j^{\tau+1}A_\tau$, $r_j=\abs{\alpha_j}(1-\varepsilon)$, and $N(c)$ the number of tokens until the next reveal, with $P(N=k)=(1-\varepsilon)^{k-1}\varepsilon$. Then $A_\tau=1$ for $\tau<N$ and $0$ from $N$ onward, so $\E_cA_\tau=(1-\varepsilon)^\tau$.

\paragraph{Magnitude.} Since $\abs{K_j(c,\tau)}=\abs{\alpha_j}^{\tau+1}A_\tau$,
\begin{equation}
 a^{(j)}=\abs{\alpha_j}\sum_{\tau=1}^{H}r_j^\tau
 =\abs{\alpha_j}\,\frac{r_j(1-r_j^H)}{1-r_j}.
 \label{eq:a}
\end{equation}

\paragraph{Lag coherence.} For one mode, the lag coherence is
\begin{equation}
\rho_\tau^{(j)} :=\frac{\E_c\abs{\frac1H\sum_{\tau=1}^{H} K_j(c,\tau)}}
{\E_c\frac1H\sum_{\tau=1}^{H}\abs{K_j(c,\tau)}},
\end{equation}
whose denominator is $a^{(j)}/H$. Summing the Jacobian over lags,
\begin{equation}
\sum_{\tau=1}^{H}K_j(c,\tau)=\alpha_j\,S_{j,N(c)},\qquad
 S_{j,N}:=\sum_{\tau=1}^{\min(N-1,H)}\alpha_j^\tau,
 \label{eq:lag-sum}
\end{equation}
where $S_{j,N}$ is the update accumulated over the steps until the state is next revealed. The lag coherence is therefore
\begin{equation}
\rho_\tau^{(j)} =\frac{\abs{\alpha_j}\E_c\abs{S_{j,N}}}{\abs{\alpha_j}\sum_{\tau=1}^{H}r_j^\tau} = \frac{(1-r_j)\E_c\abs{S_{j,N}}}{r_j(1-r_j^H)},
 \label{eq:rho-tau}
\end{equation}
where the expectation over the reveal time has the closed form
\begin{equation}
\E_c\abs{S_{j,N(c)}}=\sum_{k=1}^{H}(1-\varepsilon)^{k-1}\varepsilon\,\abs{S_{j,k}}
 +(1-\varepsilon)^H\abs{S_{j,H+1}}.
 \label{eq:middle}
\end{equation}

\paragraph{Context coherence.} The mean of the accumulated response is
\begin{equation}
\E_cS_{j,N(c)}=\sum_{\tau=1}^{H}\alpha_j^\tau(1-\varepsilon)^\tau =\sum_{\tau=1}^{H}\left(r_je^{\mathrm{i}\theta_j}\right)^\tau,
\qquad
\abs{\E_cS_{j,N(c)}}=\frac{r_j \abs{1 - (r_j e^{\mathrm{i}\theta_j})^H}}{ \abs{1 - r_j e^{\mathrm{i}\theta_j}}}.
 \label{eq:mean-sum}
\end{equation}
The context coherence
\begin{equation}
\rho_c^{(j)} :=\frac{\abs{\E_c\sum_{\tau=1}^{H} K_j(c,\tau)}}
 {\E_c\abs{\sum_{\tau=1}^{H} K_j(c,\tau)}}
 =\frac{\abs{\E_cS_{j,N}}}{\E_c\abs{S_{j,N}}}
 = \frac{r_j \abs{1 - (r_j e^{\mathrm{i}\theta_j})^H}}{\abs{1 - r_j e^{\mathrm{i}\theta_j}}\,\E_c\abs{S_{j,N}}}
 \label{eq:rho-c}
\end{equation}
follows by dividing Eq.~\eqref{eq:mean-sum} by Eq.~\eqref{eq:middle}.

\paragraph{Long-window limit.} \label{app: long-window}
Since $|\alpha_j| = \lambda |c_j| < 1$, the accumulated responses are bounded,
$|S_{j,k}| \le |\alpha_j|/(1-|\alpha_j|)$, so the boundary term
$(1-\varepsilon)^H |S_{j,H+1}|$ in Eq.~\eqref{eq:middle} and the factors $r_j^H$ in
Eqs.~\eqref{eq:a}, \eqref{eq:rho-tau}, \eqref{eq:mean-sum} and~\eqref{eq:rho-c} vanish geometrically as $H\to\infty$. The
magnitude and the mean accumulated response become
\begin{equation}
  a^{(j)} = \frac{|\alpha_j|\, r_j}{1-r_j}, \qquad
  \E_c S_{j,N} = \frac{r_j e^{\mathrm{i}\theta_j}}{1-r_j e^{\mathrm{i}\theta_j}},
  \label{eq:a_ES_inf}
\end{equation}
and, substituting $S_{j,k} = \alpha_j(1-\alpha_j^{k-1})/(1-\alpha_j)$ into Eq.~\eqref{eq:middle},
\begin{equation}
  \E_c |S_{j,N}|
  = \frac{\varepsilon\, |\alpha_j|}{|1-\alpha_j|}
    \sum_{k=0}^{\infty} (1-\varepsilon)^k \bigl|1-\alpha_j^{k}\bigr| .
  \label{eq:ES_abs_inf}
\end{equation}
The coherences are $\rho^{(j)}_\tau = (1-r_j)\,\E_c|S_{j,N}|/r_j$ and
$\rho^{(j)}_c = r_j / \bigl(|1-r_j e^{\mathrm{i}\theta_j}|\,\E_c|S_{j,N}|\bigr)$.
For general $\theta_j$, the modulus $|1-\alpha_j^k|$ prevents a closed form for
Eq.~\eqref{eq:ES_abs_inf}, but it cancels from their product and from the
singular value of Eq.~\eqref{eq:ranking},
\begin{equation}
  \rho^{(j)}_\tau \rho^{(j)}_c = \frac{1-r_j}{|1-r_j e^{\mathrm{i}\theta_j}|}, \qquad
  \sigma^{(j)} := a^{(j)}\,\rho^{(j)}_\tau \rho^{(j)}_c
        = |\alpha_j|\,\bigl|\E_c S_{j,N}\bigr|
        = \frac{|\alpha_j|\, r_j}{|1-r_j e^{\mathrm{i}\theta_j}|} ,
  \label{eq:weight_inf}
\end{equation}
so only the split of $\sigma^{(j)}$ between lag and context coherence requires the series.
The split is closed when $\alpha_j$ is real. Then $1-\alpha_j^k > 0$ for all $k$,
so $S_{j,k}$ has the sign of $\alpha_j$ for every $k$ and
$\E_c|S_{j,N}| = |\E_c S_{j,N}|$. For a static latent
($\theta = 0$), $\E_c|S_{j,N}| = r_j/(1-r_j)$ and
$\rho^{(j)}_\tau = \rho^{(j)}_c = 1$: the singular value is pure magnitude. For an
alternating mode ($\theta_j = \pi$), $\E_c|S_{j,N}| = r_j/(1+r_j)$,
$\rho^{(j)}_c = 1$ and $\rho^{(j)}_\tau = (1-r_j)/(1+r_j)$: the context
coherence is perfect, and the lag coherence carries the whole loss from rotation.
For intermediate $\theta_j$ both coherences are below one, with their product
fixed by Eq.~\eqref{eq:weight_inf}.

\paragraph{Singular-value switching.}
For the Bayes filter in belief units, the latents are independent and share
only the reveal, so the lens is block diagonal over latents (Section~\ref{sec: jlens_factors}) and
each latent's mode $j$ contributes a degenerate pair of singular values equal to
$\sigma^{(j)}$. Singular values of different latents therefore cross rather than avoid
each other, and the top direction of the lens switches from one latent to the
other exactly where their singular values are equal. Compare two latents, labeled $0$ and $1$ as in Fig.~\ref{fig: crossing}, on mode
$j=1$, which carries the largest singular value here. We fix
$j=1$ and index by latent $i$ instead: $\alpha_i = \lambda_i c_1 e^{\mathrm{i}\theta_i}$,
$r_i = (1-\varepsilon)|\alpha_i|$, and $\sigma_i$ is the corresponding singular value. Here $c_1$ is the diffusion factor of mode $j=1$, shared by both latents.
By Eq.~\eqref{eq:weight_inf}, the switch occurs when
\begin{equation}
  \frac{|\alpha_0|^2}{|1-r_0 e^{\mathrm{i}\theta_0}|}
  = \frac{|\alpha_1|^2}{|1-r_1 e^{\mathrm{i}\theta_1}|} .
  \label{eq:switch}
\end{equation}
Take latent~0 static ($\theta_0 = 0$) with its persistence $\lambda_0$ swept,
and latent~1 fixed. Let $\kappa := \sigma_1/(1-\varepsilon)
= |\alpha_1|^2/|1-r_1 e^{\mathrm{i}\theta_1}|$. Then Eq.~\eqref{eq:switch} is the
quadratic $|\alpha_0|^2 + (1-\varepsilon)\kappa\,|\alpha_0| - \kappa = 0$,
whose positive root gives the critical persistence
\begin{equation}
  \lambda_0^* = \frac{1}{2c_1}
  \Bigl[\sqrt{(1-\varepsilon)^2\kappa^2 + 4\kappa} - (1-\varepsilon)\kappa\Bigr] .
  \label{eq:lambda_star}
\end{equation}
This closed form needs $\theta_0 = 0$; otherwise the equation is quartic and has no simple root.
The singular value $\sigma_0$ rises monotonically from $0$ with $\lambda_0$. The static
latent therefore takes the top direction for $\lambda_0 > \lambda_0^*$ and loses
it below, and a switch exists within the sweep whenever
$\sigma_0|_{\lambda_0 = 1} > \sigma_1$. The rotating latent's handicap enters only
through $|1-r_1 e^{\mathrm{i}\theta_1}|$, the inverse of its product of coherences up to
the factor $1-r_1$, which is smallest for $\theta_1 = 0$ and grows with
$|\theta_1|$ up to $\pi$. A faster-turning latent is thus overtaken at a lower $\lambda_0^*$.

\section{Transformer Training Methods} \label{app: transformer}

Each model is a standard decoder-only transformer with causal self-attention and rotary position embeddings (RoPE; \citealp{su2024roformer}) applied to the queries and keys only during attention. The vocabulary size $V$ is defined by the
HMM: $V = 8^3 + 1 = 513$ for three latents on rings of $8$ (every combination of revealed places, plus the null token) and
$V = (1 + 2 \cdot 2^3)\, 2^3 = 136$ for the context sweep (null or plain or keyed reveal of three $\pm 1$ latents, times the
three key bits).

We write $\bm h_{\ell, t} \in \R^d$ for the residual stream at position $t$ before block $\ell + 1$: $\bm h_{0,t}$ is the token
embedding, $\bm h_{1,t}$ the stream between the two blocks, and $\bm h_{2,t}$ the output of the last block (so $L=2$), which the final LayerNorm
and $\bm W_U$ read. The J-lens is read at $\ell = 0$ and $\ell = 1$.

We trained each model with cross-entropy on next-token prediction, with AdamW \citep{loshchilov2019decoupled} for $20{,}000$ steps with learning rate $10^{-3}$ and weight decay $0.01$, on uniform batches of $64$ sequences. We compared it with two references: the Bayes filter, which is the optimal
predictor, and a memoryless predictor that uses the stationary distribution of the process.

\subsection{Data for all sweeps}

For each sweep and seed we draw $500{,}000$ training sequences of $192$ tokens from the respective HMM process, and $200$ test sequences. All sweeps use a reveal probability of
$\varepsilon = 0.2$. We ran 8 model seeds per task, jittering $\lambda$, $q$, $\varepsilon$ and $\psi$ for each seed by independent uniform offsets of at most $0.05$ around a fixed mean. For the magnitude sweep, we set the number of states per latent $n=8$, velocity $v=0$, diffusion probability $q = 1/4$ and persistence $\lambda_i = 0.8, 0.9, 0.95$. For the lag sweep, we set $n=8$, $v=4,2,0$, $q = 1/4$, $\lambda = 0.9$. For the context coherence sweep we set $n=2$, $v=0$, $q=0$, $\lambda = 0.95$, and key means $\psi_i = 0, 0.75, 1$.

These parameter sweeps selectively modulate magnitude, lag coherence and context coherence:

\paragraph{Magnitude sweep.} We only differ $\lambda_i$, hence all $S_N > 0$ and therefore $\mathbb{E}_c |S_N| = \mathbb{E}_c S_N$ and so $\rho_\tau = 1$, and likewise $\rho_c = 1$.  Recall that
\begin{equation}
a_i = \lambda_i c_1 \frac{r_i (1-r_i^H)}{1-r_i}
\end{equation}
with $r_i = \lambda_i c_1(1-\varepsilon)$, giving magnitude parameters of $0.82,1.23,1.50$ for $\lambda = 0.8,0.9,0.95$.

\paragraph{Lag sweep.} We have that $\lambda,c_1, \varepsilon$ are shared, hence the magnitude is fixed across all latents. Only $\theta$ varies, which varies $\alpha_i = \lambda c_1 e^{\mathrm{i}\theta_i}$, and with $S_k = \alpha_i \frac{1 - \alpha_i^{k-1}}{1 - \alpha_i}$, we have that $\rho_\tau = 1, 0.35, 0.24$ for $\theta_i = 0, \pi/2, \pi$. For $\rho_c$, the sign of $S_N$ is shared across all contexts for $\theta = 0, \pi$ and so $\rho_c = 1$. For $\pi/2$, the value is slightly less, at $\rho_c = 0.945$. 

\paragraph{Context sweep.} For the context sweep, we modified $\psi_i$ only. This keeps $\rho_\tau$ and magnitude $a$ fixed, while $\rho_c = \psi$, as $\mathbb{E}_c [W_{t-1}S_N] = \psi \mathbb{E}_c S_N$ and $\mathbb{E}|W_{t-1}S_N| = \mathbb{E}_cS_N$. 

\section{Computing alignment of J-lens vectors and latent state representations}
\label{methods: computing alignment}

We compute the Jacobian lens on $1000$ sampled sequences from the task HMM after a 32-token burn-in period, for sequences of length $192$. We computed the Jacobian only on the tokens that revealed latent states, totaling about $19{,}000$ sources. We summed over $H=64$ lags. For the transformer, the output is the token scores, using $\bm W_U$ scaled by the final LayerNorm gain, centered over tokens. The Jacobian was computed with forward-mode autodifferentiation and validated with a split-halves analysis. For the Bayes filter, the input is separated by latent, and the output is the predicted distribution over each latent separately as well. 

We computed the singular value decomposition $\bm J_H = \bm U\bm\Sigma\bm V^*$ using standard linear algebra routines in Python. The left singular vectors $\bm U$ are the output directions that the Jacobian writes, while the right singular vectors $\bm V$ are the directions that the model reads. They are ordered by the singular values on the diagonal of $\bm\Sigma$.

\section{Alignment between J-lens vectors and latent vectors}\label{app: alignment}

We identify the relationship between the J-lens vectors and the latent state features by identifying an encoding and decoding subspace for each latent state mode $\phi_j(X^{(i)})$ (the Fourier mode $j$ of latent state $i$), and computing the overlap between the J-lens vector and these subspaces. Specifically, we take $n_s$ sequences and, at the last time step of each, record $p$ latent state patterns (e.g. $\Re\,\phi_1(X^{(1)})$ and $\Im\,\phi_1(X^{(1)})$, the real and imaginary parts of the first Fourier mode of latent state $1$). We stack the centered residual activity vectors into $\bm Z \in \mathbb{R}^{n_s \times d}$ and the centered patterns into $\bm\Phi \in \mathbb{R}^{n_s \times p}$, which has two columns (real and imaginary parts) for each latent mode. These are the patterns on which encoding and decoding regression is done. 

We compute the encoding direction of each pattern independently via
\begin{equation}
    \bm E = \mathrm{diag}(\bm\Phi^\top\bm\Phi)^{-1}\bm\Phi^\top\bm Z,
\end{equation}
which is a map from patterns to activity; the diagonal means that each pattern is regressed independently, rather than jointly. For a decoding analysis, we reverse the process, computing
\begin{equation}
    \bm D_i = (\bm Z^\top\bm Z + \nu \bm I)^{-1}\bm Z^\top\bm M_i,
\end{equation}
where $\nu$ is a ridge penalty and $\bm M_i\in\R^{n_s\times 2}$ holds the Bayes filter beliefs of latent $i$, for example $\E[\phi_1(X^{(i)}_t)\mid Y_{\leq t}]$, split into real and imaginary parts.

On the output side, the patterns are directly readable from the tokens. Each token represents a particular state of all latents, so we can directly align the Fourier mode patterns broadcast to the output singular vectors. That is, each output singular vector is in $\mathbb{R}^V$, and for each Fourier mode of each latent we construct a vector in $\mathbb{R}^{V}$ by setting the entry of each token to $\cos(2\pi x/n)$ (or $\sin$), where $x$ is the state of that latent in the token. Then, the encoding analysis simply comes down to projection of the J-lens singular vector onto the plane spanned by each pattern individually. 

We compute alignment by measuring how much a unit vector of the J-lens lies in the span of a group of encoding or decoding directions, which in practice are taken to be the vectors corresponding to a single mode of a latent state (typically two vectors corresponding to the sine and cosine directions). We define the overlap of a unit lens vector $\bm v$ by projecting it onto the span of the $k_i$ reference vectors corresponding to the relevant feature:
\begin{equation}
    O_i(\bm v) = \norm{\bm Q_i^\top\bm v},
\end{equation}
where the columns of $\bm Q_i$ are an orthonormal basis for the $k_i$ reference vectors, constructed with a QR decomposition. Under this construction, a random unit direction in $d$ dimensions has an expected squared projection of $k_i/d$ onto a $k_i$-dimensional subspace, making the chance level of overlap $\sqrt{k_i/d}$, with $k_i=2$ for Fourier modes and $d = 128$. We assign each J-lens direction to the latent with the largest overlap. 

\section{Validating singular value switching theory}\label{app: switching}

The switching analysis (Appendix~\ref{app: long-window}) makes two predictions. First, in the large-$H$ limit the singular value of each latent's first Fourier mode is $\sigma_i = \abs{\alpha_i} r_i / \abs{1 - r_i e^{\mathrm{i}\theta_i}}$ (Eq.~\eqref{eq:weight_inf}). Second, for a static latent~0 with swept persistence $\lambda_0$ and a fixed latent~1, the two singular values cross, and the top J-lens direction switches from latent~1 to latent~0, at the persistence $\lambda_0^*$ of Eq.~\eqref{eq:lambda_star}. We test both predictions on the Bayes filter, and the second on transformers trained across the switch.

\paragraph{Task.} We use the toy-model HMM with two unkeyed latents on rings of $n = 8$ states, diffusion probability $q = 1/4$ and reveal probability $\varepsilon = 0.2$, which gives $V = 8^2 + 1 = 65$ tokens. Latent~0 is static ($v_0 = 0$) and its persistence $\lambda_0$ is swept. Latent~1 has fixed persistence $\lambda_1 = 0.95$ and velocity $v_1 \in \{1, 2, 3, 4\}$, that is, angular velocity $\theta_1 = \pi v_1 / 4 \in \{\pi/4, \pi/2, 3\pi/4, \pi\}$, which are all the nonzero turns a ring of $8$ allows, up to direction.

\paragraph{Theory.} The predicted singular values are Eq.~\eqref{eq:weight_inf}, evaluated for each latent on mode $j = 1$. The predicted switch point is Eq.~\eqref{eq:lambda_star}, which depends only on latent~1 and the shared parameters, not on $\lambda_0$. The series for $\E_c\abs{S_{j,N}}$ (Eq.~\eqref{eq:ES_abs_inf}), needed for the individual coherences but not for $\sigma_i$, is truncated at $200$ terms.

\paragraph{Measuring the singular values.} We compute the J-lens as in Appendix~\ref{methods: computing alignment}, on $1000$ sampled sequences of $192$ tokens with a $32$-token burn-in, using only the revealed source tokens, and sum over $H = 64$ lags. The first Fourier mode of each latent is a plane (its real and imaginary parts), so each latent owns two singular directions. We assign the J-lens directions to latents by the overlap of their output side with each latent's mode-1 pattern (Appendix~\ref{app: alignment}), and take $\sigma_i$ to be the mean of the two singular values assigned to latent $i$. For the Bayes filter the two are degenerate (Appendix~\ref{app: long-window}).

\subsection{Bayes filter}

\paragraph{Singular values.} We sweep $\lambda_0 \in [0.55, 0.95]$ in steps of $0.025$ at $\theta_1 = \pi/2$, and compare the measured $\sigma_0$ and $\sigma_1$ with Eq.~\eqref{eq:weight_inf}. The same sampled sequences are reused at every $\lambda_0$, so that sampling noise does not vary along the sweep.

\paragraph{Switch point.} For each $\theta_1$ we measure $\sigma_0$ and $\sigma_1$ at seven values of $\lambda_0$, spaced $0.02$ apart and centered on Eq.~\eqref{eq:lambda_star}. We locate the crossing by linear interpolation of $\log(\sigma_0/\sigma_1)$, which increases monotonically with $\lambda_0$, and compare it with Eq.~\eqref{eq:lambda_star}.

\subsection{Transformer}

\paragraph{Training.} For each $\theta_1$ we train transformers (Appendix~\ref{app: transformer}) on a grid of $11$ persistences $\lambda_0 = \lambda_0^* + \Delta$, with $\Delta$ from $-0.24$ to $0.16$ in steps of $0.04$ and $\lambda_0^*$ from Eq.~\eqref{eq:lambda_star}. The grid is wide because the transformer's switch need not sit at the filter's. At every grid point we train $5$ models from independent seeds, $220$ models in total, and compute each model's J-lens at layers $\ell = 0$ and $\ell = 1$ on $1000$ freshly sampled sequences. The transformer's lens is in token-score units, so its singular values are on a different scale from the filter's (Section~\ref{sec: jlens_factors}). We therefore compare only the location of the crossing with the theory.

\paragraph{Estimating the crossing.} For each model we take the difference $\Delta\sigma = \sigma_0 - \sigma_1$ and compute its mean and standard deviation over seeds at each grid point. We then smooth these with a Savitzky--Golay filter (a least-squares quadratic over a window of $5$ grid points). The crossing $\hat\lambda_0^*$ is where the smoothed mean $\mu$ changes sign, found by linear interpolation on that grid step, and we propagate the spread across seeds into an uncertainty on the crossing as
\begin{equation}
    \delta\hat\lambda_0^* = \frac{s(\hat\lambda_0^*)}{\abs{\mathrm d\mu/\mathrm d\lambda_0}(\hat\lambda_0^*)} .
    \label{eq:crossing_error}
\end{equation}
We compare $\hat\lambda_0^* \pm \delta\hat\lambda_0^*$ at each layer with Eq.~\eqref{eq:lambda_star}.

\section{Supplementary Results}

\begin{figure}[!htbp]
    \centering
    \includegraphics[width=0.72\linewidth]{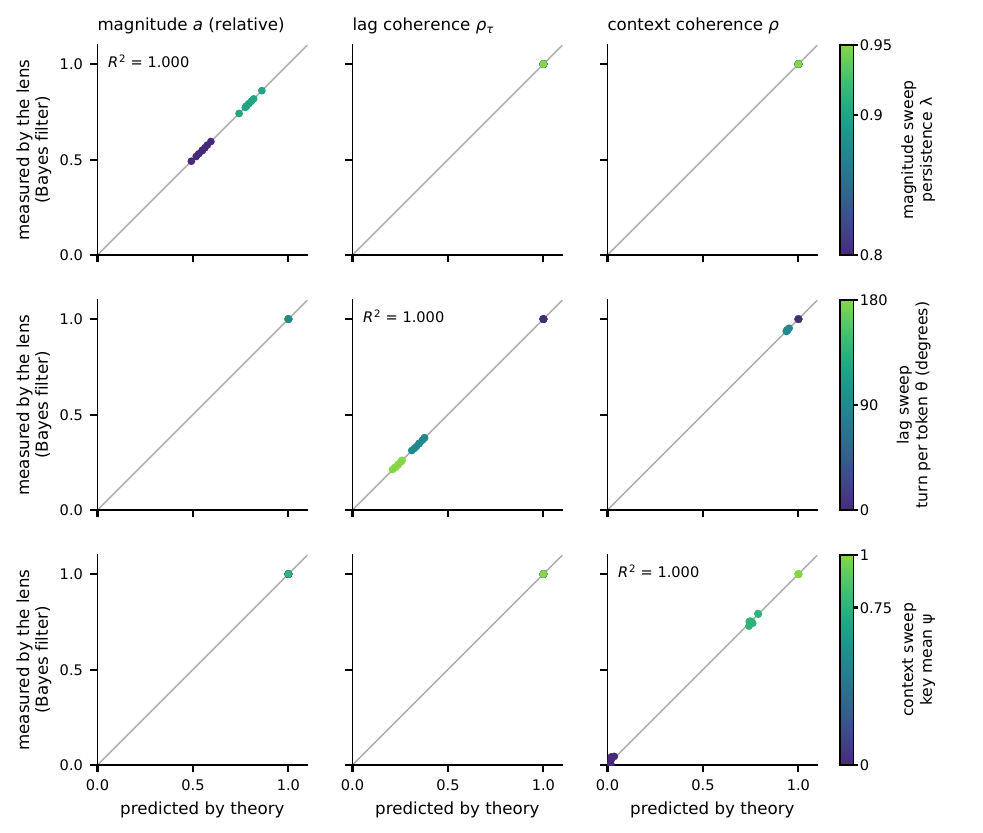}
    \caption{Validation of the J-lens theory for the Bayes filter. X-axis denotes the value predicted by the theory (Eq.~\eqref{eq:bayes-factors}), while the y-axis denotes the value measured from the lens, for magnitude $a$, lag coherence $\rho_\tau$ and context coherence $\rho_c$ (columns) in each sweep (rows, colored by $\lambda$, $\theta$ or $\psi$). There is only one visible data point for the fixed parameters in each sweep because the HMM is designed to fix these metrics in the Bayes filter.}
    \label{fig: filter_validation}
\end{figure}

\begin{figure}[!htbp]
    \centering
    \includegraphics[width=0.72\linewidth]{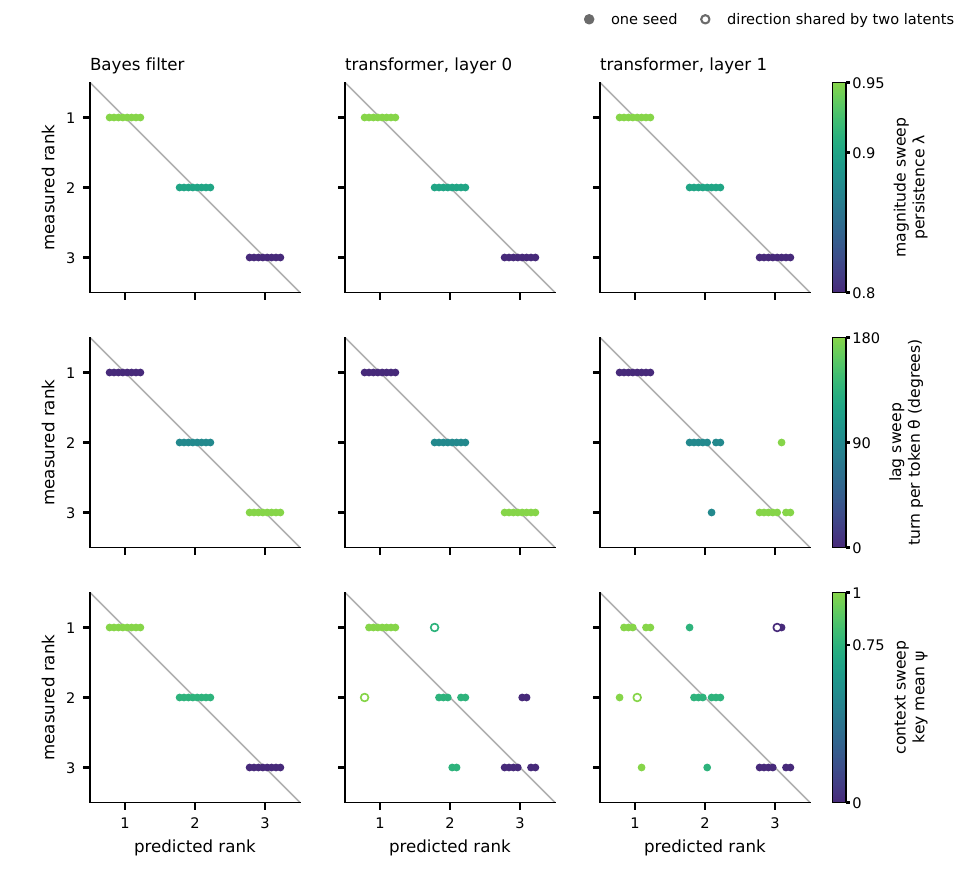}
    \caption{Validation of the J-lens theory rank prediction for the Bayes filter and the transformer. X-axis denotes the predicted rank order of the singular vectors corresponding to each particular latent (colored by their task parameter $\lambda$, $\theta$ or $\psi$), while the y-axis denotes the measured rank in the J-lens.}
    \label{fig: transformer_validation}
\end{figure}

\begin{figure}[!htbp]
    \centering
    \includegraphics[width=1\linewidth]{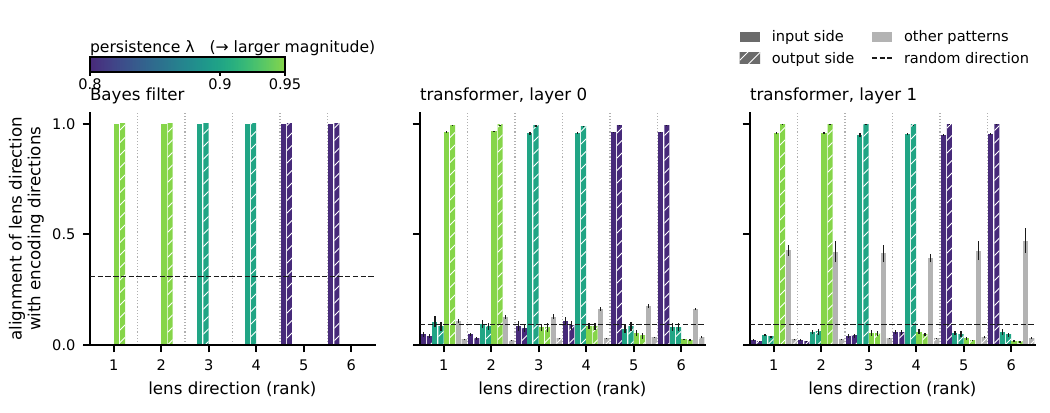}
    \caption{\textbf{Alignment of the J-lens directions with the encoding directions in the magnitude sweep.} Alignment $O_i$ of the J-lens singular vectors (solid: input side $\bm v_i$, hatched: output side $\bm u_i$), ranked by singular value, with the encoding directions $\bm E$ of each latent (colored by persistence $\lambda$) and with other patterns of the HMM (gray). Columns show the Bayes filter belief state (left), transformer layer 0 (middle) and layer 1 (right). Dashed line: alignment of a random direction.}
    \label{fig: lens_magnitude}
\end{figure}

\begin{figure}[!htbp]
    \centering
    \includegraphics[width=1\linewidth]{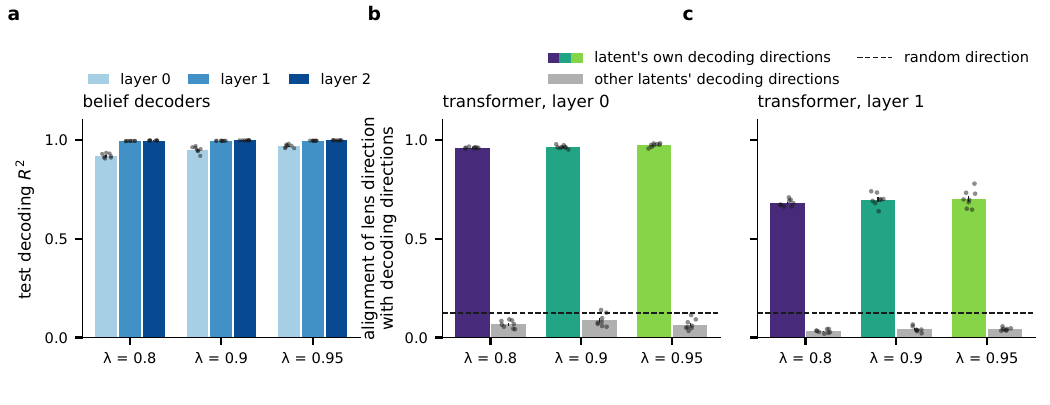}
    \caption{Alignment of the J-lens vectors with the decoding vectors for each latent in the magnitude sweep. 
    \textbf{(a)} Test decoding accuracy ($R^2$) for the Bayes filter belief state $m_t$ on each layer of the transformer, for each latent in the sweep. 
    \textbf{(b, c)} Alignment $O_i$ of the J-lens directions $\bm v_i$ with each latent's own decoding directions $\bm D_i$ (colored) and with the other latents' decoding directions (gray), for transformer layer 0 (b) and layer 1 (c). }
    \label{fig: decoding_magnitude}
\end{figure}

\begin{figure}[!htbp]
    \centering
    \includegraphics[width=1\linewidth]{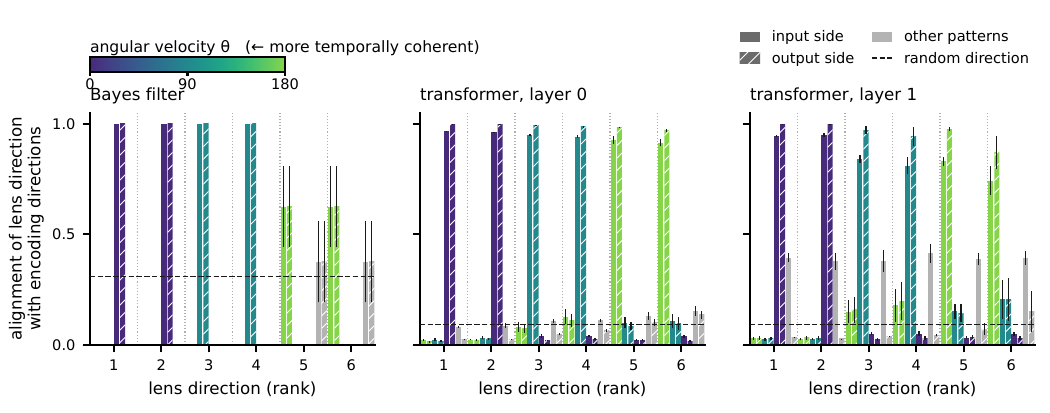}
    \caption{\textbf{Alignment of the J-lens directions with the encoding directions in the lag coherence sweep.} Latents are colored by angular velocity $\theta$. Panels as in Fig.~\ref{fig: lens_magnitude}.}
    \label{fig: lens_lag}
\end{figure}

\begin{figure}[!htbp]
    \centering
    \includegraphics[width=1\linewidth]{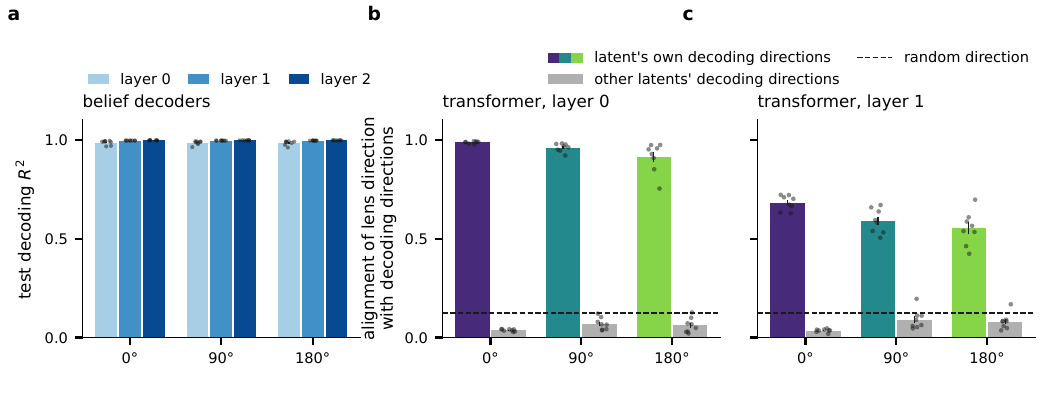}
    \caption{Alignment of the J-lens vectors with the decoding vectors for each latent in the lag coherence sweep. 
    \textbf{(a)} Test decoding accuracy ($R^2$) for the Bayes filter belief state $m_t$ on each layer of the transformer, for each latent in the sweep. 
    \textbf{(b, c)} Alignment $O_i$ of the J-lens directions $\bm v_i$ with each latent's own decoding directions $\bm D_i$ (colored) and with the other latents' decoding directions (gray), for transformer layer 0 (b) and layer 1 (c). }
    \label{fig: decoding_lag}
\end{figure}

\begin{figure}[!htbp]
    \centering
    \includegraphics[width=1\linewidth]{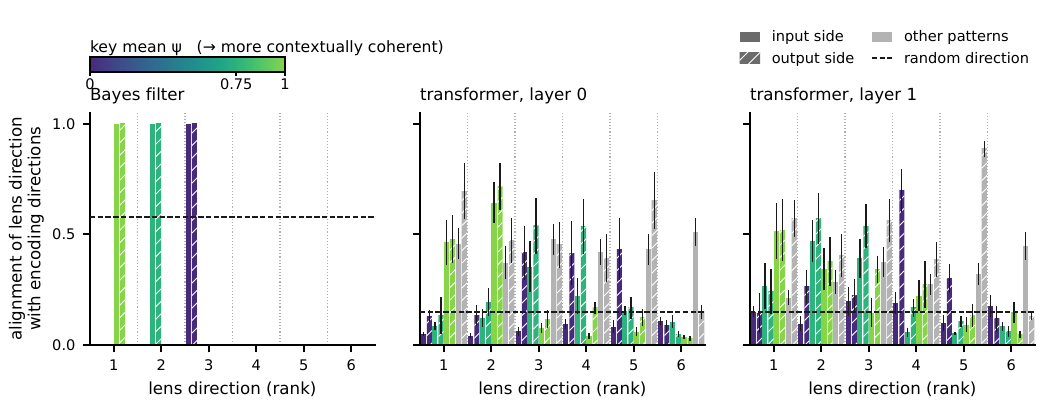}
    \caption{\textbf{Alignment of the J-lens directions with the encoding directions in the context coherence sweep.} Latents are colored by key mean $\psi$. Panels as in Fig.~\ref{fig: lens_magnitude}.}
    \label{fig: lens_context}
\end{figure}

\begin{figure}[!htbp]
    \centering
    \includegraphics[width=1\linewidth]{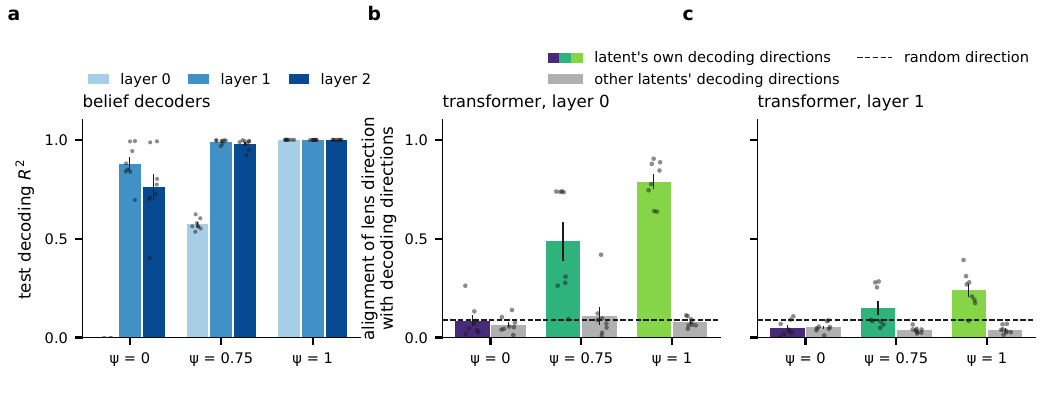}
    \caption{Alignment of the J-lens vectors with the decoding vectors for each latent in the context coherence sweep. 
    \textbf{(a)} Test decoding accuracy ($R^2$) for the Bayes filter belief state $m_t$ on each layer of the transformer, for each latent in the sweep. 
    \textbf{(b, c)} Alignment $O_i$ of the J-lens directions $\bm v_i$ with each latent's own decoding directions $\bm D_i$ (colored) and with the other latents' decoding directions (gray), for transformer layer 0 (b) and layer 1 (c). }
    \label{fig: decoding_context}
\end{figure}

\begin{figure}[!htbp]
\centering
\includegraphics[width=\linewidth]{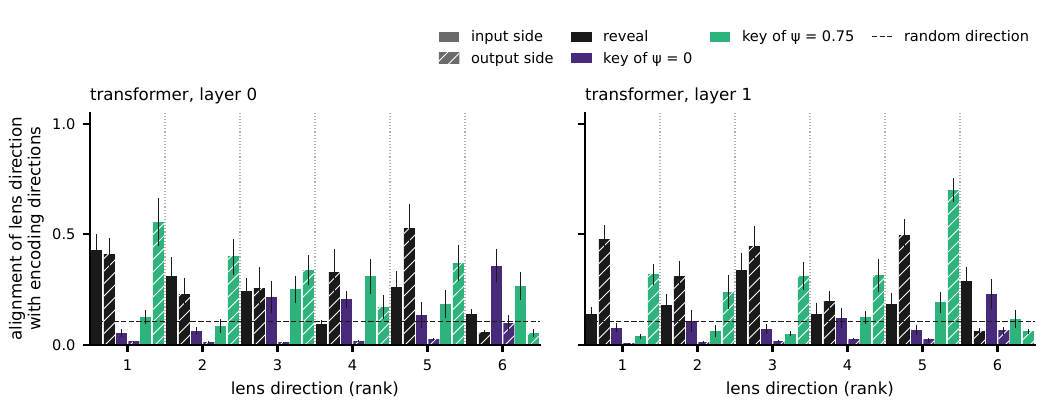}
\caption{\textbf{J-lens direction alignment with other features of the HMM that are not the latent states in the context sweep.} Black denotes the reveal state $B_t$: whether or not the latent state is revealed in the token, and the colors denote the key state $W_{t-1}$ (whether the latent state is flipped) for each key mean $\psi$.}
\label{fig: other_patterns_context}
\end{figure}

\end{document}